\documentclass{article}

\usepackage{arxiv}

\usepackage[utf8]{inputenc} 
\usepackage[T1]{fontenc}    
\usepackage{hyperref}       
\usepackage{url}            

\usepackage{booktabs}       
\usepackage{amsfonts}       
\usepackage{nicefrac}       
\usepackage{microtype}      
\usepackage{xcolor}

\usepackage{graphicx}
\usepackage{nicematrix,enumitem}
\usepackage{natbib}

\usepackage{algorithm}
\usepackage{algorithmic}

\usepackage{amsmath}
\usepackage{amssymb}
\usepackage{mathtools}
\usepackage{amsthm}

\usepackage{subcaption}
\usepackage{mathtools}
\usepackage{array}

\definecolor{darkblue}{rgb}{0.0,0.0,0.65}
\definecolor{darkred}{rgb}{0.68,0.05,0.0}
\definecolor{darkgreen}{rgb}{0.0,0.29,0.29}
\definecolor{darkpurple}{rgb}{0.47,0.09,0.29}
\definecolor{darkpink}{rgb}{0.9,0.17,0.31}

\hypersetup{
   colorlinks = true,
   citecolor  = darkblue,
   linkcolor  = darkred,
   filecolor  = darkgreen,
   urlcolor   = darkpink,
 }

\newcommand{\real}{\mathbb{R}}
\newcommand{\E}{\mathbb{E}}
\newcommand{\calA}{\mathcal{A}}

\newcommand{\calO}{\mathcal{O}}
\newcommand{\calS}{\mathcal{S}}
\newcommand{\calE}{\mathcal{E}}

\mathtoolsset{showonlyrefs}
\newcommand{\regret}{\mathfrak{R}}

\DeclareMathOperator{\rd}{RD}

\DeclareMathOperator{\id}{ID}

\usepackage{bm}

\newcommand{\inner}[2]{\langle #1 \,, #2 \rangle}
\DeclareMathOperator*{\argmin}{arg\,min}

\newcommand{\btheta}{\bm{\theta}}

\newcommand{\frakD}{\mathfrak{D}}
\newcommand{\Learn}{\textsc{Learn}}

\newcommand{\conC}{\psi}
\newcommand{\conD}{\phi}
\newcommand{\conE}{\kappa}
\newcommand{\conF}{\nu}
\newcommand{\conB}{\xi}

\theoremstyle{plain}
\newtheorem{theorem}{Theorem}[section]

\newtheorem{lemma}[theorem]{Lemma}
\newtheorem{corollary}[theorem]{Corollary}
\theoremstyle{definition}
\newtheorem{definition}[theorem]{Definition}
\newtheorem{assumption}[theorem]{Assumption}
\theoremstyle{remark}
\newtheorem{remark}[theorem]{Remark}

\usepackage[textsize=tiny]{todonotes}

\title{{\sc LEARN}: An Invex Loss for Outlier Oblivious Robust Online Optimization}

\author{%
  Adarsh Barik\\
  Institute of Data Science\\
  National University of Singapore\\
  Singapore, 117602 \\
  \texttt{abarik@nus.edu.sg} \\
  \And
  Anand Krishna\\
  Department of Electrical and Computer Engineering\\
  National University of Singapore\\
  Singapore, 117583 \\
  \texttt{akr@nus.edu.sg} \\
  \And
  Vincent Y. F. Tan\\
  Department of Mathematics\\
  Department of Electrical and Computer Engineering\\
  National University of Singapore\\
  Singapore, 119077 \\
  \texttt{vtan@nus.edu.sg} \\
}

\begin{document}

\maketitle

\begin{abstract}
We study a robust online convex optimization framework, where an adversary can introduce outliers by corrupting loss functions in an arbitrary number of rounds $k$, unknown to the learner. Our focus is on a novel setting allowing unbounded domains and large gradients for the losses without relying on a Lipschitz assumption. We introduce the Log Exponential Adjusted Robust and iNvex  (\Learn{}) loss, a non-convex (invex) robust loss function to mitigate the effects of outliers and develop a robust variant of the online gradient descent algorithm by leveraging the \Learn{} loss. We establish tight regret guarantees (up to constants), in a dynamic setting,  with respect to the uncorrupted rounds and conduct experiments to validate our theory. Furthermore, we present a unified analysis framework for developing online optimization algorithms for non-convex (invex) losses, utilizing it to provide regret bounds with respect to the \Learn{} loss, which may be of independent interest.
\end{abstract}

\section{Introduction}
\label{sec:intro}

In the mercurial era of digital information, the proliferation of online data streams has become ubiquitous, with applications spanning from financial markets and healthcare to e-commerce and social media~\citep{hendricks2016detecting,zhang2015health,fedor2019realtime,kejariwal2015analytics}. The surge in real-time data generation across these diverse fields has underscored the pressing need for adaptive optimization strategies. These strategies need to navigate dynamic and uncertain environments, especially as organizations rely more on online decision-making processes. In this context, the demand for efficient and robust optimization techniques in the presence of outliers has become more pronounced.
\paragraph{Outlier robustness} Pioneering works in robust statistics~\citep{huber1964robust,tukey1974fitting} emphasized mitigating outliers' impact on algorithm performance. In online optimization originally designed for uncorrupted streams, outliers pose challenges leading to suboptimal solutions~\citep{feng2017outlier}. This motivates developing outlier-robust online methodologies, an area of recent focus~\citep{diakonikolas2022streaming,van2021robust,bhaskara2020robust,feng2017outlier}.
\paragraph{Online convex optimization (OCO)}
The OCO framework~\citep{hazan2023introduction,orabona2023modern} offers a useful mathematical framework for studying sequential predictions, especially in the presence of adversarial interventions. The adversary chooses a deterministic sequence of convex losses $f_t$ for each $t\in \{ 1,2,\ldots, T \}$, where $T$ denotes the time horizon. In every round, the learner picks an action $\theta_t$ in a convex set $\Theta$ for which she suffers the loss $f_t(\theta_t)$ for the round $t$ and additionally gets to observe $f_t$. The learner's goal is to minimize the cumulative loss, i.e., $\sum_{t=1}^Tf_t(\theta_t)$, and her performance is evaluated using the notion of regret.
\paragraph{OCO with robustness} The OCO framework serves as a valuable starting point for studying algorithms that dynamically adapt to incoming data points while considering potential outliers' influence.
Drawing inspiration from the model proposed by \citet{van2021robust}, this work considers the OCO framework where the adversary is allowed to choose any $\calS\subseteq [T]$ of rounds to be `clean' and inject outliers into the remaining rounds. This adversarial influence is exerted by tampering with the side information $s_t$ integral to constructing the loss function $f_t$. Considering the example of the online ridge regression problem, where $s_t=(x_t, y_t)$ represents the incoming data point for round $t$, and the loss is defined as $f_t(s_t,\theta_t)=\frac{1}{2}(\inner{\theta_t}{x_t}-y_t)^2$, the adversary can corrupt either $x_t$, $y_t$ or both to corrupt the round $t$. The corrupted round is considered an outlier. We use $c_t$ to denote the uncorrupted side information and $o_t$ for the corrupted side information. The learner's performance is measured using regret with respect to only the clean rounds, and we refer to this as clean regret. In this paper, we explore the notion of dynamic regret -- a performance measure suitable for dynamic environments characterized by evolving data and shifting optimal decision parameters. Dynamic regret entails adapting strategies to compete effectively with evolving comparators, thereby mitigating the impact of changing environments on decision-making processes.  Let $\btheta \coloneqq (\theta_t)_{t\in [T]}$ be a sequence of learner's actions and  $(\omega^*_t)_{t\in [T]}$ be a sequence of benchmark points such that $\omega_t^*=\argmin_{\theta \in \Theta} f_t(s_t, \theta)$. We also define $\theta_t^*=\argmin_{\theta \in \Theta} f_t(c_t, \theta)$. Clearly, $\theta_t^* = \omega_t^*$ for clean rounds. Then, the clean dynamic regret is defined to be
 \begin{align}
 \label{eq: robust dynamic regret}
 \regret_{\rd}^T(\btheta, \calS) \coloneqq \sum_{t\in \calS} \big( f_t(s_t, \theta_t)-f_t(s_t, \theta^*_t) \big)~.
 \end{align}
 Our goal is to develop robust online algorithms to achieve sublinear regret with respect to $T$ and achieve optimal dependency on the number of outliers $k$, which is unknown.

\begin{figure}[!ht]
    \centering
    \begin{subfigure}{0.5\textwidth}
    \centering
    \includegraphics[scale=0.7]{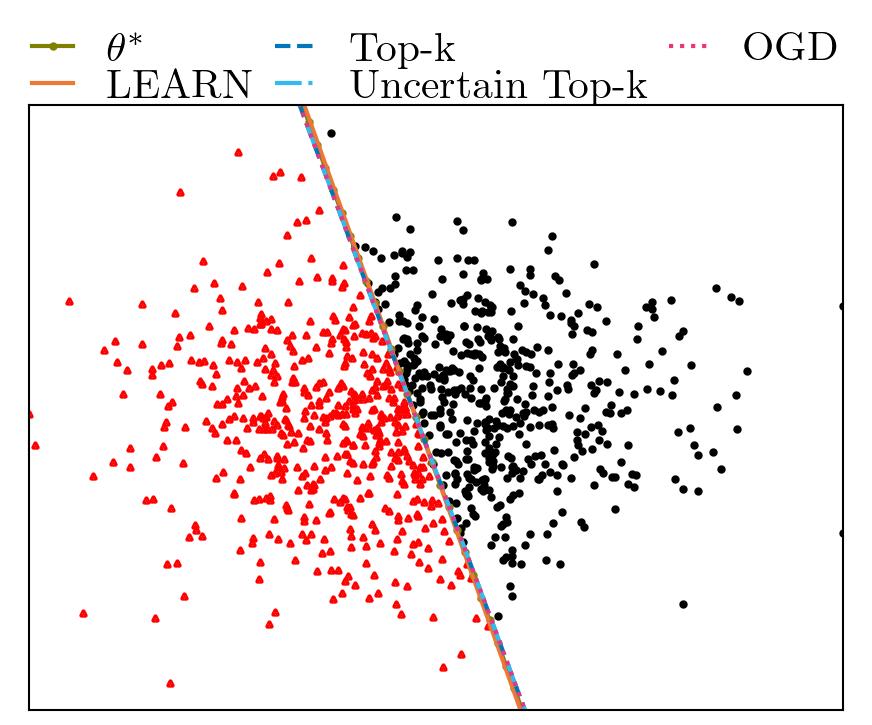}
        \caption{\label{fig: svm clean decision} $k = 0$ }
    \end{subfigure}%
    \begin{subfigure}{0.5\textwidth}
    \centering
    \includegraphics[scale=0.7]{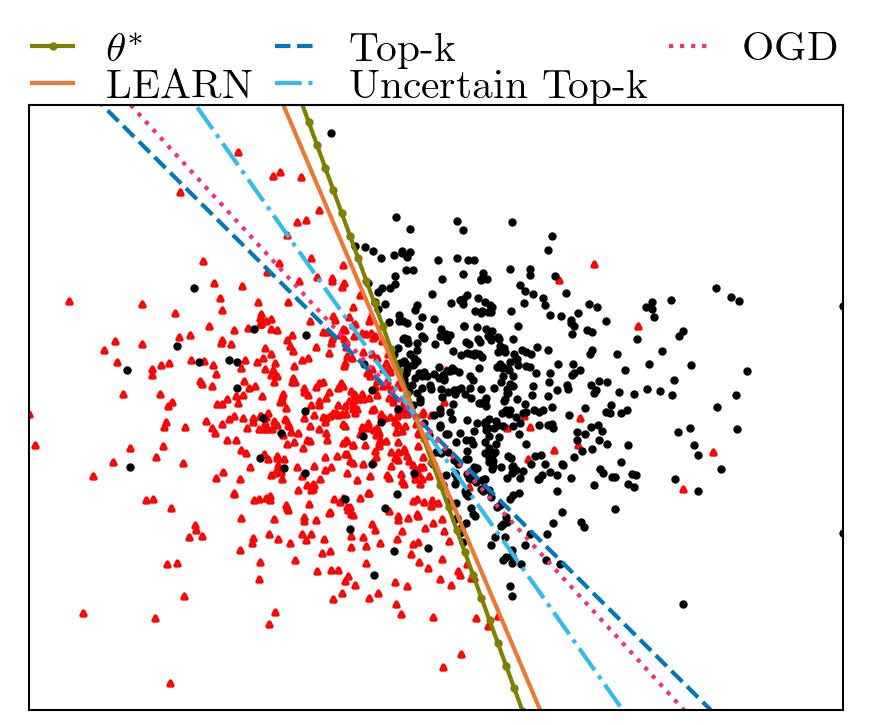}
        \caption{\label{fig: svm ktwothird decision} $k = T^{\frac{2}{3}}$ }
    \end{subfigure}%
    \caption{\label{fig: svm decision boundary}
    We compared \Learn{} (Algorithm~\ref{alg:outlier robust OGD}) and baselines (Section~\ref{sec: experimental validation main}) on binary classification with hinge loss and sequential data. Figure~\ref{fig: svm clean decision} shows the true decision boundary with no outliers and all the methods recover it exactly. Figure~\ref{fig: svm ktwothird decision} shows how decision boundaries for baseline methods deviate from the ground truth in the presence of outliers while \Learn{} remains robust to outliers.
    }
\end{figure}


This work focuses on $m$-strongly convex losses. Diverging from \citet{van2021robust}'s static model, we introduce novel OCO techniques for outlier scenarios when the number of outliers $k$ is unknown. Unlike their work assuming known $k$, and in contrast to common bounded domain/Lipschitz loss assumptions \citep{hazan2023introduction,orabona2023modern}, our approach allows unbounded domains and non-Lipschitz losses. Following \citet{jacobsen2023unconstrained}, we permit arbitrarily large gradients away from $f_t$'s minimizer, extending applicability to outlier scenarios.

\begin{table*}[!ht]
  \caption{Best known upper bounds on the dynamic regret of strongly convex loss functions $f_t$ in the presence of outliers (with respect to time-horizon $T$ and number of outliers $k$). We use $V_T \coloneqq \sum_{t=1}^T \| \theta_t^* - \theta_{t+1}^* \|$ to denote the path-length of benchmark points.  
  }
  \label{tab:SOTA results}
  \centering
  \begin{tabular}{lll}
    \toprule
         Setting & Best-known upper bound & Our upper bound \\
    \midrule
    Bounded domain & $\calO(\sqrt{V_T T} + k)$ & $\calO(\sqrt{V_T T} + k)$ \\
    Unbounded domain & Not Available & $\calO(\sqrt{V_T T} + V_T + \sqrt{T \log T} + k)$ \\
    \bottomrule
  \end{tabular}
\end{table*}

\subsection{Main contributions}

In this paper, we study a crucial question within the robust OCO framework: \vspace{-1em}
\begin{quote}
\emph{Can we achieve sublinear clean dynamic regret in $T$ with optimal dependency on the outlier count $k$, without its prior knowledge, when the adversary can corrupt an arbitrary subset of rounds?}
\end{quote} \vspace{-1em}
We answer the aforementioned question in the affirmative with a linear dependence on the number of outliers, which is known to be unavoidable in the robust OCO setting~\citep{van2021robust}. Our main contributions in this work are summarized below:

\begin{enumerate}
\item \textbf{Introduction of Robust Invex Loss:} At the core of our work lies the introduction of a robust, invex loss function, which we refer to as the \Learn{} loss~(refer to Equation~\eqref{eq: robust loss}). The limitations of convex loss functions in handling outliers are well studied in the literature~\citep{chen2022online}. As alternative solutions, researchers have investigated non-convex robust losses~\citep{barron2019general}. However, given the evolving landscape of mathematical studies in non-convex optimization, this transition poses challenges in the analysis. Striking a judicious balance, our \Learn~loss addresses the outlier-handling requirements and lends itself to rigorous theoretical analysis. %
%
%
%
\item \textbf{Unified analysis framework for robust online gradient descent algorithm:} We introduce \Learn, a robust variant of online gradient descent leveraging the outlier-handling \Learn{} loss (Algorithm ~\ref{alg:outlier robust OGD}). We provide a unifying analytical framework for analyzing Algorithm~\ref{alg:outlier robust OGD}'s regret, instrumental in establishing guarantees for our invex loss (Section~\ref{sec: reg analysis of robust loss}). This framework extends to geodesically convex~\citep{wang2023online} and pseudo-convex~\citep{zhang2018dynamic} losses, enhancing the applicability of our approach. We substantiate \Learn{}'s robustness via online SVM experiments (Fig.~\ref{fig: svm decision boundary}, Appendix~\ref{subsec:online-svm}), where it consistently predicts boundaries closely matching ground truth despite outliers, aligning with the theoretical guarantees (details are provided in Appendix~\ref{subsec:online-svm}).
\item \textbf{Sublinear clean dynamic regret in bounded domain:} In addition to analyzing regret for invex losses, we study the clean regret for strongly convex loss functions within the robust OCO framework (Section~\ref{subsec: clean regret analysis}). Our algorithm can handle potentially unbounded gradients while remaining agnostic to the number of outliers.
Our upper bound exhibits sublinear growth, specifically as $\calO(\sqrt{V_T T} + k)$, where optimal path variation $V_T$ is defined as $V_T = \sum_{t=1}^T \| \theta_t^* - \theta_{t+1}^* \|$.  The bound is optimal with respect to both $T$ and $k$ and matches the lower bound presented in~\citet{van2021robust} for the static case (Also see Appendix~\ref{sec: extension of van erven's result}).
\item \textbf{An expert framework with outliers and sublinear clean dynamic regret:} Extending dynamic regret bounds to unbounded domains is non-trivial, as shown by \citet{jacobsen2023unconstrained} who used an expert framework. Introducing outliers poses new challenges hindering a straightforward extension to our setup. We address this by constructing a novel expert framework with $N$ experts to handle outliers, incurring $\mathcal{O}(\sqrt{T\log N} + k)$ expert regret (Lemma~\ref{lem: bound on expert regret}). This expert framework construction/analysis could be of independent interest. Utilizing it, we derive clean dynamic regret bounds. Our bound is $\mathcal{O}(\sqrt{V_T T} + V_T + \sqrt{T\log T} + k)$. This matches existing bounds without outliers in $T, V_T$ dependence. Ours is the first work providing dynamic regret bounds for unbounded domains with outliers, to our knowledge.
\end{enumerate}
Table~\ref{tab:SOTA results} presents a concise summary of our results and contrasts them with the best existing results. In the bounded domain, the best-known upper bound presented in the table follows from the analysis of~\citet{van2021robust}. We note that \citet{van2021robust}'s results pertain to the static setting. In Appendix~\ref{sec: extension of van erven's result}, we present a straightforward extension of their result to the dynamic setting. However, unlike \citet{van2021robust}'s approach, our algorithm does not require prior knowledge of $k$ and is capable of handling unbounded gradients. In the unbounded domain, to the best of our knowledge, our work presents the first-known results. We further validate our theoretical results through comprehensive numerical experiments conducted in Section~\ref{sec: experimental validation main}.

\paragraph{Related works:} Our starting point is the robust OCO framework with outliers proposed by \citet{van2021robust}, though their analysis is limited to static regret with known outlier counts. The foundational works of \citet{huber1964robust} and \citet{tukey1974fitting} motivate the importance of handling outliers from a robust statistics perspective. Incorporating robust loss functions like Huber, Tukey's biweight, etc., has been explored by \citet{barron2019general,belagiannis2015robust} to enhance algorithm resilience against outliers. We build upon the expert framework approach of \citet{jacobsen2023unconstrained} for unconstrained domains, extending it to handle outliers. Our robust approach aligns with the broader literature on robust optimization in theoretical computer science and machine learning problems, as comprehensively surveyed by \citet{diakonikolas2023algorithmic}. A detailed section on the related works appears in the Appendix~\ref{sec:related work}.

\section{Preliminaries and notation}
\label{sec:prelim and notation}
We consider an online learning protocol delineated by the subsequent framework. At each iteration $t\in [T]:=\{1, 2,\ldots, T\}$, the learner picks an action $\theta_t\in \Theta$ where $\Theta\subseteq \real^d$ is a convex set. Note that we provide results for both bounded and unbounded $\Theta$. Subsequently, the adversary reveals an $m$-strongly convex, non-negative loss function $f_t(\theta) \coloneqq f_t(s_t, \theta)$ where $s_t$ denotes the side information for round $t$, resulting in the learner incurring a loss of $f_t(s_t,\theta_t)$. Note that the side information $s_t$ is also revealed along with the function $f_t$ to the learner in every round $t$.
For fixed $s_t$, $\nabla f_t(s_t,\theta_t)$ denotes the gradient of $f_t(\theta)$ with respect to $\theta$ at $\theta = \theta_t$. While we assume throughout that the gradient $\nabla f_t(s_t,\theta_t)$ exists, our analysis still holds when $f_t$ is not smooth.
We allow for the potential existence of large gradients by following the relaxed gradient assumption of \citet{jacobsen2023unconstrained} with unknown non-negative constants $G$ and $L$, i.e.,
\begin{align}
    \label{eq: grad inequality}
    \| \nabla f_t(s_t, \theta_t) \| \leq G + L \| \theta_t - \omega_t^* \|~,
\end{align}
where $\omega_t^*$ acts as the reference point. The round $t$ is said to be an outlier round if $s_t$ is corrupted 
and the corruption can be arbitrary. 
We use $\calS \subseteq [T]$ to denote the clean rounds. The maximum loss incurred by the learner is assumed to be smaller than $B$ in the clean rounds, i.e., $ \max_{t \in \calS} f_t(s_t,\theta_t) \leq B$, and the loss incurred by the learner may be unbounded in outlier rounds. Next, we provide a formal definition of invex functions, which generalize the concept of convex functions:
\begin{definition}[Invex functions]
    A differentiable function $g: \real^d \to \real$ is called a $\zeta$-invex function on set $\Theta$ if for any $\vartheta_1, \vartheta_2 \in \Theta$ and a vector-valued function $\zeta: \Theta \times \Theta \to \real^d$,
    \begin{align}
        g(\vartheta_2) \geq g(\vartheta_1) + \inner{\nabla g(\vartheta_1)}{\zeta(\vartheta_2, \vartheta_1)}
    \end{align}
    where $\nabla g(\vartheta_1)$ is the gradient of $g$ at $\vartheta_1$.
\end{definition}
\paragraph{\Learn{} loss} In this work, we introduce a non-convex robust loss function, Log Exponential Adjusted Robust and iNvex (\Learn{}) loss $g: \Theta \rightarrow \mathbb{R}$ that transforms the output of a convex function $f:\Theta \to \mathbb{R}$ as follows
\begin{align}
    \label{eq: robust loss}
    g(\theta)=-a\log\bigg(\exp\Big(-\frac{1}{a}f(\theta)\Big)+b\bigg)~,
\end{align}
where $a,b>0$ are constants that can be tuned. Robust losses have been widely used in various real-world problems to reduce the sensitivity of algorithms to outliers~\citep{barron2019general}. The structure of \Learn{} shares similarities with the log-sum-exp function, used previously for robustness~\citep{lozano2013minimum}. To understand how \Learn{} loss works to mitigate its susceptibility to outliers, we examine its behavior by comparing how $g$ resembles the square loss $f(r)=r^2$ where $r$ denotes the residual, and the minimum is attained at  $0$.
The \Learn{} loss $g$ closely follows the behavior of the function $f$ in the vicinity of the origin and gradually levels off as we move away from it (Figure~\ref{fig: loss comparison} in Appendix~\ref{sec: comparison of robust losses}). Importantly, the minimizers for both $g$ and $f$ are the same.
Observe that for \Learn{} loss, the norm of the gradient monotonically decreases with the distance from the minimum. This reduces the adverse influence of outliers during a gradient descent style update. This property, known as the {\em redescending property}~\citep{hampel2011robust} in the M-estimation literature,  enhances the robustness of the model. Additionally, it is worth noting that empirical studies have shown that non-convex losses offer better robustness when compared to the convex alternatives~\citep{maronna2006robust}.

\section{Robust online convex optimization}

\label{subsec: analysis robust dynamic regret}

In this section, we develop \Learn{}, a robust variant of the online gradient descent algorithm aimed at minimizing regret within the robust (OCO) framework. We will consistently refer to this algorithm simply as \Learn{}. This algorithm leverages the previously introduced \Learn{} loss.

 Given a convex set $\Theta\subseteq \real^d$, the algorithm starts by picking an arbitrary action $\theta_1\in \Theta$. For each round $t \in [T]$, the learner is revealed a convex loss function $f_t$ in conjunction with the side-information $s_t$. Note that the adversary is capable of corrupting the side information for any $k$ out of the $T$ rounds and the learner does not see whether $s_t$ is corrupted. To mitigate the challenge presented by the corruption introduced by the adversary, our algorithm leverages \Learn{} loss. The algorithm proceeds to construct $g_t$ based on the original loss function $f_t$ (Step~\ref{step: construct g}) and
 executes the subsequent action through a projected gradient descent update on the robust loss $g_t$.

 \newcommand{\alglinelabel}{%
  \addtocounter{ALC@line}{-1}
  \refstepcounter{ALC@line}
  \label
}

\floatname{algorithm}{Algorithm}
\begin{algorithm}[H]
\caption{\Learn{}}\label{alg:outlier robust OGD}
\begin{algorithmic}[1]
     \STATE \textbf{Initialize: } $\theta_1 \in \Theta$
     \FOR{$t = 1, \ldots, T$}
        \STATE \textbf{Play: } $\theta_t$
        \STATE \textbf{Observe: } Side information $s_t$ and strongly convex loss function $f_t(s_t, \theta)$
        \STATE \textbf{Construct: } For $a, b > 0$,\\ $g_t(\theta) \coloneqq g_t( s_t, \theta) = - a \log( \exp( - \frac{1}{a} f_t(s_t, \theta) ) + b )$ \alglinelabel{step: construct g}
        \STATE \textbf{Update: } $\theta_{t+1} \gets \Pi_{\Theta}(\theta_t - \alpha_t \nabla g_t( \theta_t))$\alglinelabel{step: choose action} \COMMENT{$\Pi_{\Theta}(\cdot)$ is the projection on convex set $\Theta$}
    \ENDFOR
\end{algorithmic}
\end{algorithm}

In the context of encountering outliers, it is relevant to analyze the clean dynamic regret as defined in Equations~\eqref{eq: robust dynamic regret}, a task that we undertake in  Section~\ref{subsec: clean regret analysis}. However, we go beyond the confines of this initial exploration and examine the regret guarantees that pertain to the \Learn{} loss itself. This exploration holds intrinsic interest, as it unveils our ability to extend the regret analysis to encompass a broader spectrum of invex losses beyond our originally introduced \Learn{} loss denoted by $g$.

To streamline the presentation of the results in this paper, we introduce the following notations:
\begin{alignat}{5}
\label{eq: constant parameters}
 \conC &\coloneqq G + \max\Big( \frac{mL}{2ab}, \frac{4a^2L}{m^2b} \Big),\;\;  & \conD  &\coloneqq \frac{1}{b} \max\Big( \frac{m}{2a}, \frac{4a^2}{m^2} \Big),\;\;  & \conE &\coloneqq \frac{1}{b} \max\Big( \frac{m}{2a},
 \frac{2a}{m} \Big), \\
 \conF &\coloneqq \frac{1}{b} \max\Big( a, \frac{1}{a} \Big),\;\; & \conB &\coloneqq 1 + b \exp\Big( \frac{B}{a} \Big)~.
\end{alignat}
The learner has the flexibility to fine-tune these constants by carefully selecting parameters $a$ and $b$. For corrupted rounds, we introduce an additional quantity $\delta_{\calS} \coloneqq \max_{t \in [T] \setminus \calS} \| \omega_t^* - \theta^*_t \| $.
This quantity measures the adversary's capability to influence the optimal action of round $t$, denoted by $\theta_t^*$. Recall that for uncorrupted rounds, $\omega_t^* = \theta_t^*$.

\subsection{Regret analysis of the \Learn{} loss}
\label{sec: reg analysis of robust loss}
Here, as an independent analysis, we derive regret bounds associated with invex losses $g_t$ in isolation, agnostic to the presence of outliers.\footnote{We omit the effect of the side information $s_t$ for this analysis by defining $g_t(\theta) \coloneqq g_t(s_t, \theta)$ for a fixed $s_t$.} In particular, we study the dynamic regret $\regret_{\id}^T$ associated with the invex losses $g_t$ defined as
\begin{align}
    \label{eq: invex dynamic regret}
    \regret_{\id}^T(\btheta) \coloneqq \sum_{i=1}^T g_t(\theta_t) - \sum_{i=1}^T g_t(\omega_t^*)~.
\end{align}
Note that by construction $\omega_t^* = \arg\min_{\theta \in \Theta} g_t(\theta)$.

We introduce a versatile framework designed to bound the dynamic regret in the context of online invex optimization. Notably, this framework is not confined to \Learn{} loss; it extends seamlessly to encompass other invex losses, including but not limited to geodesically convex loss, weakly pseudo-convex loss, and similar cases.  In this setup, we will assume that $\Theta$ is a closed and bounded convex set.
We further assume that all the convex losses are bounded, i.e., $ f_t(s_t, \theta) \leq B < \infty$ for all $t \in [T]$ (even when $s_t$ is corrupted).  We use the invexity of the \Learn{} loss and combine it with the results of our unified analysis framework (See Appendix~\ref{sec: dynamic regret for invex losses}) to prove an upper bound on the robust dynamic regret. To this end, we state and prove the invexity of $g_t$ constructed at each step $t\in[T]$ by \Learn{} (Step~\ref{step: construct g}). We denote
    \begin{align}
        \label{eq: eta definition}
        \eta_t \coloneqq \eta_t(s_t, \theta_t) = \frac{\exp(-\frac{1}{a} f_t(s_t, \theta_t))}{b + \exp(-\frac{1}{a} f_t(s_t, \theta_t))}~.
    \end{align}

\begin{lemma}
    \label{lem: g_t is invex}
    The robust loss $g_t(\theta) \coloneqq g_t(s_t, \theta)$ constructed at each round $t\in [T]$ is a $\zeta_t$-invex function, i.e., for all $\vartheta_1, \vartheta_2 \in \real^d$, there exists a vector-valued function $\zeta_t:\real^d \times \real^d \to \real^d$ such that
    \begin{align}
        g_t(\vartheta_2) \geq g_t(\vartheta_1) + \inner{\nabla g_t( \vartheta_1)}{\zeta_t(\vartheta_2, \vartheta_1)} \; .
    \end{align}
    Moreover, for actions $\theta_t$ generated in Algorithm~\ref{alg:outlier robust OGD} and $\omega_t^*$ as defined in Equation~\eqref{eq: invex dynamic regret}, we may take $\zeta_t$ as  $\zeta_t(\omega_t^*, \theta_t) = \frac{1}{\eta_t} \left( \omega_t^* - \theta_t\right)$.
\end{lemma}


The following theorem establishes an upper bound on the robust dynamic regret $\regret_{\id}^T$ when the learner employs Algorithm~\ref{alg:outlier robust OGD} and executes the actions $\btheta$ by constructing $\zeta_t$-invex robust losses $g_t$ for each $t\in[T]$.

\begin{theorem}
    \label{thm: bound on invex regret}
    Consider a sequence of $m$-strongly convex loss functions $\{f_t(s_t, \cdot)\}_{t=1}^T$. Let $\{g_t\}_{t=1}^T$ be the corresponding sequence of $\zeta_t$-invex robust losses generated via Algorithm~\ref{alg:outlier robust OGD} (Step~\ref{step: construct g}) and $\btheta$ be the sequence of actions chosen by the learner using Algorithm \ref{alg:outlier robust OGD}. If the learner selects $\alpha_t = \alpha= \sqrt{ \frac{ 4D^2 + 6DV_T }{ \conC^2 T } }$,
    then the dynamic regret is 
    \begin{align}
        \regret_{\id}^T(\btheta) &\leq \conB \conC \sqrt{ ( 4D^2 + 6DV_T)T }~.
    \end{align}
 \end{theorem}
    Notice that by optimally tuning the stepsize, $\regret_{\id}^T(\btheta)$ experiences sublinear growth with respect to $T$. The accompanying constants, which are independent of $T$, can be fine-tuned by judiciously choosing parameters $a$ and $b$. It is possible to ease the constraints associated with the bounded domain and the prior knowledge of $V_T$ by adopting an expert framework while preserving a sublinear regret. However, a detailed exploration of regret bounds in this context is postponed to subsequent sections.

\subsection{Clean dynamic regret analysis}
\label{subsec: clean regret analysis}

In this section, we study how \Learn{} performs within the robust OCO framework when evaluated using the clean regret. We consider both bounded and unbounded domains and provide upper bounds tight up to constant factors.


\subsubsection{Bounded domain}

 In the following theorem, we establish upper bounds on the clean dynamic regret for $m$-strongly convex functions $\{f_t\}_{t=1}^T$ bounded in $[0,B]$ for rounds $t\in [T]$ that are uncorrupted by the adversary. We take $\| \theta \| \leq D, \forall \theta \in \Theta$. The learner determines the actions to be taken by following Algorithm~\ref{alg:outlier robust OGD}.
 With no prior knowledge on $B$, we provide an $\calO(\sqrt{V_T T}+k)$ regret bound matching the lower bound presented in~\citet{van2021robust} (extended to a dynamic environment in Appendix~\ref{sec: extension of van erven's result}).

\begin{theorem}
 \label{thm: bounded domain clean dynamic regret}
    For each $t \in [T]$, consider a sequence of $m$-strongly convex loss functions $\{f_t(s_t, \cdot)\}_{t=1}^T$. The losses during uncorrupted rounds are bounded in $[0, B]$, while the losses during corrupted rounds may be unbounded. Let 
    $\btheta$ be the sequence of actions chosen by the learner using Algorithm~\ref{alg:outlier robust OGD}. If the learner chooses $\alpha_t = \alpha = \sqrt{ \frac{ 4D^2 + 6DV_T }{\conC^2 T } } $ ,
    then
    \begin{align}
    \label{eq: clean dynamic theorem sqrt T}
       \regret_{\rd}^T(\btheta, \calS)
    &\leq  \conB \left(  \conC  \sqrt{ (4D^2 + 6DV_T) T} +   k(G \conD+ L \conE) +  k (G +  L\conD) \delta_{\calS}  \right)~.
    \end{align}
 \end{theorem}
    The behavior of $\regret_{\rd}^T(\btheta, \calS)$ grows sublinearly with respect to the total number of rounds $T$ and linear concerning the count of outlier rounds $k$.  It is crucial to underscore the dependence of the regret bound on $\delta_{\calS}$, a measure quantifying the adversary's capability to perturb the optimal action during corrupted rounds. This dependence highlights the sensitivity of the regret to the adversary's influence in the presence of outliers.  Lastly, the constants independent of $T$ within the regret expression can be fine-tuned through careful parameter choices for $a$ and $b$. This provides an avenue for further optimization, allowing the learner to perform better with meticulous parameter tuning.

\subsubsection{Unbounded domain via an expert framework}

The scenario becomes notably more intricate when dealing with an unbounded domain, introducing an additional layer of complexity exacerbated by the presence of corrupted rounds initiated by the adversary. Recent work by \citet{jacobsen2023unconstrained} have addressed the challenges associated with unbounded domains through the development of an expert framework. In alignment with this approach, we propose an expert framework tailored to handle outlier rounds.
Our framework encompasses $N$ experts, each executing an instance of Algorithm~\ref{alg:outlier robust OGD} with a fixed stepsize $\alpha^\tau > 0$ and operating within a bounded domain $D^{\tau} > 0$ where $\tau$ refers to a given expert. The learner strategically selects their action by means of a carefully constructed weighted average of the actions taken by the individual experts. As the domain is unbounded, we define $B \coloneqq \max_{t \in \calS} f_t(s_t, \theta_t)$ for this section. We present the experts' version of the \Learn{} algorithm in Algorithm~\ref{alg:outlier robust OGD expert}.


\floatname{algorithm}{Algorithm}
\begin{algorithm}[t]
\caption{\textsc{Learn with Experts}}\label{alg:outlier robust OGD expert}
\begin{algorithmic}[1]
     \STATE \textbf{Initialize: } \text{For all } $\tau \in [N], \text{ select }\theta_1^{\tau} \in \Theta^{\tau},  \text{where }\Theta^{\tau} := \{ \theta \mid \| \theta \|_2 \leq D^\tau  \} \cap \Theta $ and set $\rho_1^{\tau} = 1$
     \STATE \textbf{Define: } $Z_t \coloneqq \sum_{\tau = 1}^N \rho_t^{\tau}$
     \FOR{$t = 1, \ldots, T$}
        \STATE \textbf{Play: } $\theta_t \gets \frac{1}{Z_t} \sum_{\tau = 1}^N \rho_t^\tau \theta_t^\tau$
        \STATE \textbf{Observe: } Side information $s_t$ and $m$-strongly convex loss function $f_t$
        \FOR{$\tau= 1,2,\ldots,N$}
            \STATE \textbf{Query: } $f_t(s_t, \theta_t^{\tau})$
            \STATE \textbf{Construct: } For $a, b > 0$,\\ $g_t(\theta) \coloneqq g_t( s_t, \theta) = - a \log( \exp( - \frac{1}{a} f_t(s_t, \theta) ) + b )$\alglinelabel{step: construct g expert}
            \STATE \textbf{Update: } ${\theta_{t+1}^{\tau} \gets \Pi_{\Theta^{\tau}}(\theta_t^{\tau} - \alpha^{\tau} \nabla g_t(\theta_t^{\tau}))}$\alglinelabel{step: choose action expert}
        \ENDFOR
        \STATE \textbf{Define: }  $\widetilde{\eta}_t \coloneqq \min_{\tau \in [N]} \eta_t(s_t, \theta_t^\tau)$
        \STATE \textbf{Update: } For all $\tau \in [N]$ where $\beta > 0$, ${\rho_{t+1}^\tau \gets \rho_{t}^\tau \exp( - \beta \tilde{\eta}_t f_t(s_t, \theta_t^\tau) )}$
    \ENDFOR
\end{algorithmic}
\end{algorithm}
Next, we provide a proof sketch for bounding the clean dynamic regret $\regret_{\rd}^T$ that matches the lower bound in terms of $T$ and $V_T$  presented in~\citet{jacobsen2023unconstrained} for the setting without outliers. Consider parameters $A_{\max} \geq C \sqrt{T}$ for sufficiently large $C > 0$ and $\epsilon \geq 0$. The selection of parameters ($\alpha^\tau$ and $D^\tau$) for an expert $\tau$ is drawn from a set $\calE$, defined as:
\begin{align}
\label{eq: expert parameters}
    \calE &\coloneqq \big\{ (\alpha, D) \mid \alpha \in \calE_1, D \in \calE_2 \big\}~,
\end{align}
where $ \calE_1 \coloneqq \big\{   \alpha_i = \min\big( \frac{2^i}{\sqrt{T}}, \frac{A_{\max}}{\sqrt{T}} \big)  \big\}_{i \in \mathbb{N}}$  and $\calE_2 \coloneqq \big\{ D_j = \min\big( \frac{\epsilon 2^j}{T}, \frac{\epsilon 2^T}{T} \big)  \big\}_{j\in\mathbb{N}}$. Clearly, $N \coloneqq |\calE| \leq T\log_2 A_{\max} $. Let $\btheta^\tau \coloneqq (\theta_1^\tau, \cdots, \theta_T^\tau)$ be a sequence of actions chosen by the expert $\tau$. Observe that for any expert $\tau$,
\begin{align}
    \regret_{\rd}^T(\btheta, \calS) &= \underbrace{\sum_{t \in \calS} f_t(s_t, \theta_t^\tau) - \sum_{t \in \calS} f_t(s_t, \theta_t^*)}_{\text{Meta Regret $\regret_{\rm{M}}^T(\btheta^{\tau}, \calS)$}}  + \underbrace{\sum_{t \in \calS} f_t(s_t, \theta_t) - \sum_{t \in \calS} f_t(s_t, \theta_t^\tau)}_{\text{Expert Regret $\regret_{\rm{E}}^T(\btheta, \btheta^{\tau}, \calS)$}}
\end{align}

Taking $ \frakD \coloneqq \max_{t=1}^T \| \theta_t^* \|$, we bound $\regret_{\rd}^T(\btheta, \calS)$ in four major steps. First, if $\frakD \leq D^{\tau}$ for any specific expert $\tau$, then we establish that $\regret_{\rm{M}}^T(\btheta^{\tau}, \calS)$ is bounded (Lemma~\ref{lem: dynamic regret bounded domain}). Second, we demonstrate that  $\regret_{\rm{E}}^T(\btheta, \btheta^{\tau}, \calS)$ remains bounded when $\frakD \leq D^{\tau}$. To that end, we prove the following bound on expert regret under the presence of outliers.
\begin{lemma}
    \label{lem: bound on expert regret}
    By choosing $\beta = \sqrt{\frac{8 \log N}{T \nu^2}} $, the following regret guarantee holds with respect to any expert $\tau \in [N]$ following Algorithm~\ref{alg:outlier robust OGD expert}:
    \begin{align}
        \label{eq: bound on expert regret}
        \regret_{\rm{E}}^T(\btheta, \btheta^{\tau}, \calS) &\leq \conB  \Big( k \conF  + \conF \sqrt{\frac{T\log N}{2}}    \Big)~.
    \end{align}
\end{lemma}


Third, we establish the bound on $\regret_{\rd}^T(\btheta, \calS)$ when $\frakD \geq D_{\max} \coloneqq \max_{\tau=1}^N D^{\tau}$ (Lemma~\ref{lem: regret bound large comparator}). Finally, we show that $\regret_{\rd}^T(\btheta, \calS)$ is bounded when $\frakD \leq D_{\min} \coloneqq \min_{\tau=1}^N D^{\tau}$ (Lemma~\ref{lem: regret bound small comparator}). Combining all these together, we state the following result:
    \begin{theorem}[Informal]
    \label{thm: clean dynamic regret bound}
    For each $t\in[T]$, let $\{f_t(s_t, \cdot) \}_{t=1}^T$ be a sequence of $m$-strongly convex loss functions with possible $k$ of them corrupted by outliers. We assume that the loss for uncorrupted rounds falls within the range $[0, B]$, whereas the loss for corrupted rounds can be unbounded. Let $\btheta$ be a sequence of actions chosen by the learner using Algorithm~\ref{alg:outlier robust OGD expert} with parameters for the experts chosen from $\calE$ as defined in equation~\eqref{eq: expert parameters}. Let $\alpha^* = \sqrt{ \frac{32 \frakD^2 + 24 \frakD V_T}{T\conC^2}  }$.  Then,
        \begin{align}
        \mathfrak{R}_{\mathrm{RD}}^{T}(\boldsymbol{\theta}, \mathcal{S}) \leq \mathcal{O}(\sqrt{(\mathfrak{D}^{2}+\mathfrak{D} V_{T}) T} + k + \frakD V_T + \sqrt{T \log T})~.
        \end{align}
\end{theorem}
    Refer to Appendix~\ref{sec: bounds on the clean dynamic regret} for the formal version of Theorem~\ref{thm: clean dynamic regret bound}.

\paragraph{An overview of proof techniques} We now provide a high-level overview of our proof strategies, deferring the detailed proofs to the Appendix. Using standard methods (such as online gradient descent) in OCO, a typical bound on dynamic regret takes the following form:
\begin{align}
    \label{eq: template of dynamic regret}
    \regret_{\rd}\big( \btheta, \calS \big) &\leq \sigma_1 \frac{\| \theta_1 - \theta_1^* \|}{\alpha} + \underbrace{\sigma_2 \alpha \sum_{t \in \calS}\| \nabla f_t(s_t, \theta_t) \|}_{Q_1} + \underbrace{\sigma_3 \sum_{t \in [T] \setminus \calS} \| \nabla f_t(s_t, \theta_t) \| \| \theta_t - \theta_t^* \|}_{Q_2} \notag \\
    &\quad + \underbrace{\sigma_4 \sum_{t=1}^T \frac{\| 2\theta_{t+1} - \theta_{t}^* - \theta_{t+1}^* \| \| \theta_{t}^* - \theta_{t+1}^* \|}{\alpha}}_{Q_3}~,
\end{align}
for some positive $\sigma_i, i \in \{1,\ldots,4\}$.  Notice that $Q_1$ depends on the gradient norm for uncorrupted rounds, $Q_2$ depends on $\| \nabla f_t(s_t, \theta_t) \|$ and domain size for corrupted rounds, while $Q_3$ depends on the domain size across all rounds. When gradients and domain sizes are bounded, choosing $\alpha$ appropriately ensures $\calO(\sqrt{T})$ growth for $Q_1$ and $Q_3$, limiting $Q_2$ to $\calO(k)$ growth. Typically, in bounded domains, $Q_1$ is controlled by assuming bounded gradient norms for uncorrupted rounds, and $Q_2$ by filtering rounds with large gradient norms, ensuring small gradient norms for corrupted rounds  (e.g.,~\cite{van2021robust}). However, this requires bounded gradients for uncorrupted rounds and prior knowledge of $k$, often unavailable, leading to uncontrollable growth of $Q_1$ and $Q_2$.
We circumvent this issue by carefully constructing the \Learn{} loss function, working with $\eta_t \nabla f_t(s_t, \theta_t)$ instead of $\nabla f_t(s_t, \theta_t)$, and controlling its growth (Lemmas~\ref{lem: exp trumps poly},\ref{lem: bound on eta grad^r}). In the unbounded domain, $Q_3$ can grow arbitrarily large and we develop an expert framework to control its growth (Lemma~\ref{lem: bound on expert regret} and Appendix~\ref{sec: bounds on the clean dynamic regret}). Experts operate in chosen bounded domains with fixed step sizes, and the learner strategically combines their actions to achieve a dynamic regret bound of $\calO(\sqrt{V_T T} + V_T + \sqrt{T \log T} + k)$.


\section{Experimental validation}
\label{sec: experimental validation main}
This section presents numerical experiments conducted to substantiate our theoretical findings. We focus on two quintessential machine learning problems, online linear regression, and classification to validate our theory. In particular, we focus on online ridge regression and online support vector machine (SVM).
%
%
We compare \Learn{} with the Top-k filter algorithm (Top-k) of~\citet{van2021robust}  and vanilla online gradient descent (OGD). Note that Top-k has access to the number of outliers $k$, but \Learn{} does not know $k$. We also implemented an uncertain version of Top-k, labeled ``Uncertain Top-k'', which only has access to an estimate of $k$ fixed at $0.75k$. We conduct experiments in an unbounded domain with potentially large gradients of $m$-strongly convex loss functions. Note that none of the baselines provide theoretical guarantees for unbounded domains with outliers. All methods employ the same stepsize $\alpha = \frac{1}{\sqrt{T}}$. We would like to clarify our choice of step size in the experiments. Given that the losses are $m$-strongly convex, it is possible to set $\alpha_t = \frac{1}{mt}$ for the baselines. However, in our experiments, the value of $m$ was very small ($m = \lambda = 10^-4$), which results in a large step size. Consequently, the expected regret bound of $\calO(\frac{\log T}{m})$ becomes quite large. A numerical comparison revealed that using $\alpha = \frac{1}{\sqrt{T}}$ yields smaller regret for the baselines. Therefore, we opted for this more favorable step size in our experimental setup. We report the results for the clean dynamic regrets in Figure~\ref{fig: Clean regret for svm and lin reg}, averaged across $30$ independent runs with the standard errors. We provide additional experimental details in Appendix~\ref{subsec:online-svm} and Appendix~\ref{subsec:online-linear-regression}.

\begin{figure*}[!ht]
    \centering
    \begin{subfigure}[t]{0.25\textwidth}
    \includegraphics[scale=0.4]{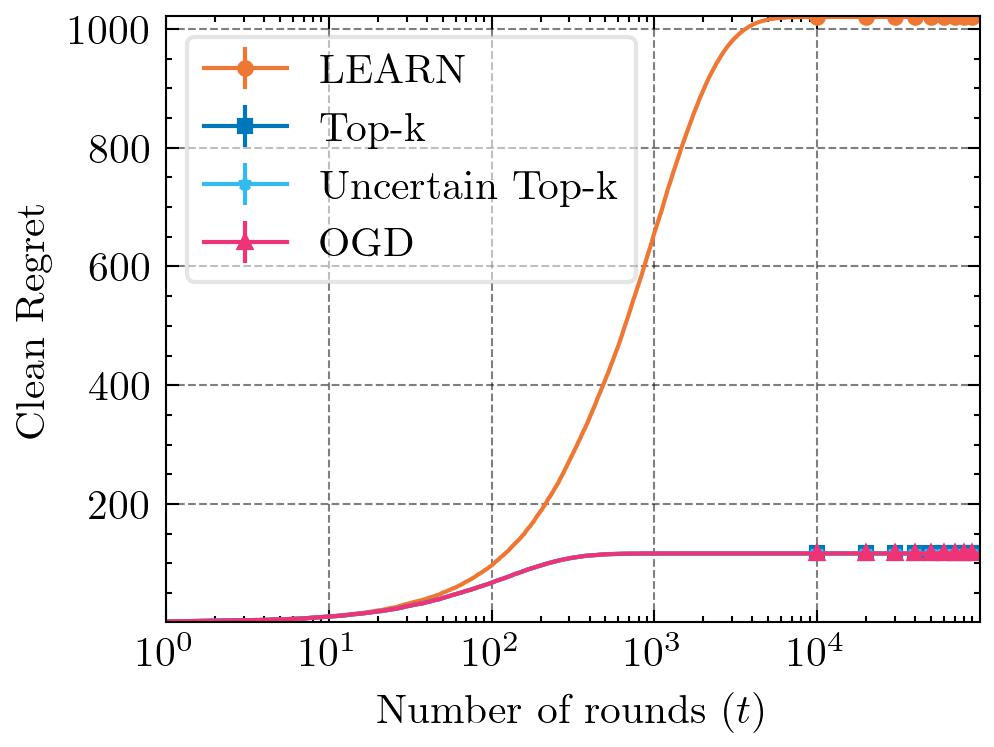}
        \caption{\label{fig: lin reg kcfr k0} $k = 0$ }
    \end{subfigure}%
    \begin{subfigure}[t]{0.25\textwidth}
    \includegraphics[scale=0.4]{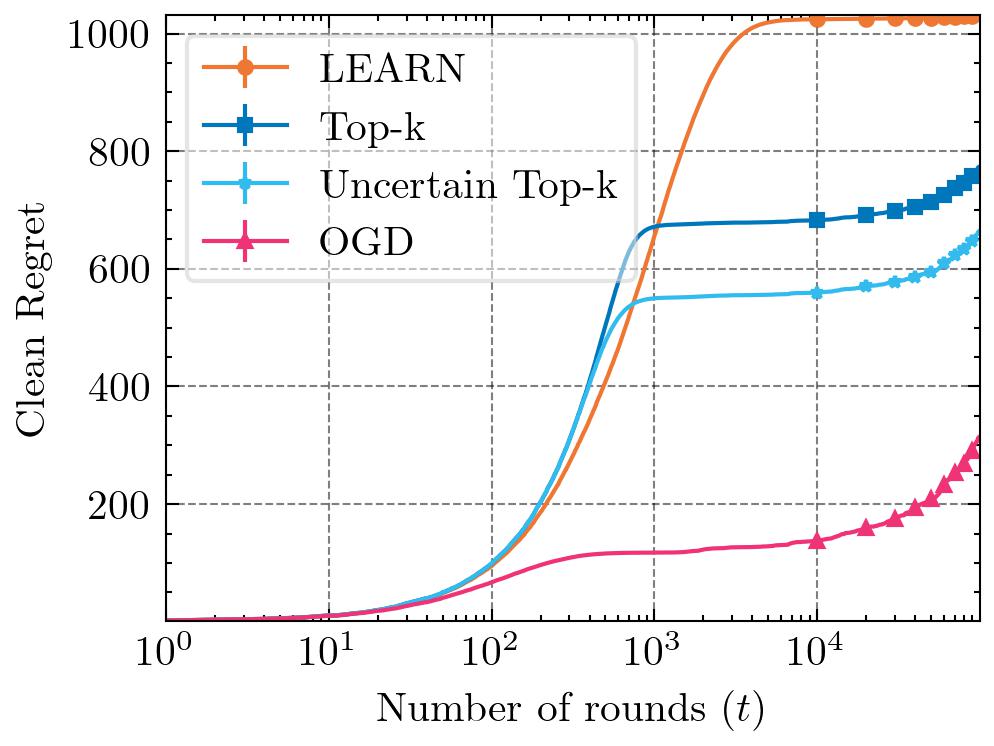}
        \caption{\label{fig: lin reg kcfr ksqrt} $k = \sqrt{T}$ }
    \end{subfigure}%
    \begin{subfigure}[t]{0.25\textwidth}
    \includegraphics[scale=0.4]{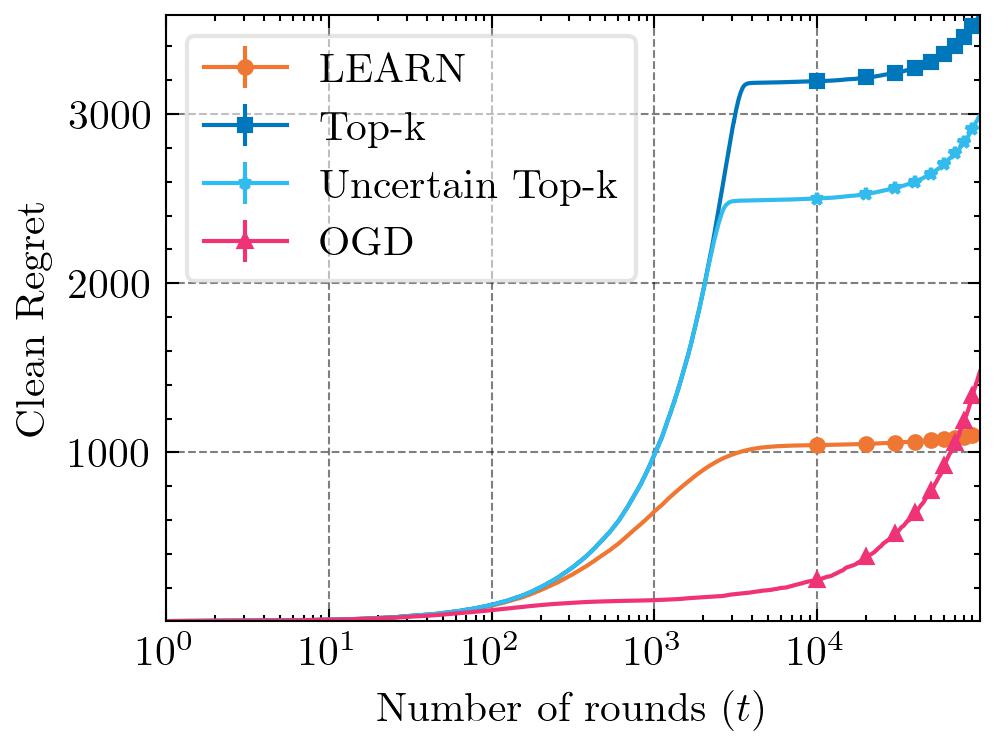}
        \caption{\label{fig: lin reg kcfr ktwothird} $k = T^{\frac{2}{3}}$ }
    \end{subfigure}%
    \begin{subfigure}[t]{0.25\textwidth}
    \includegraphics[scale=0.4]{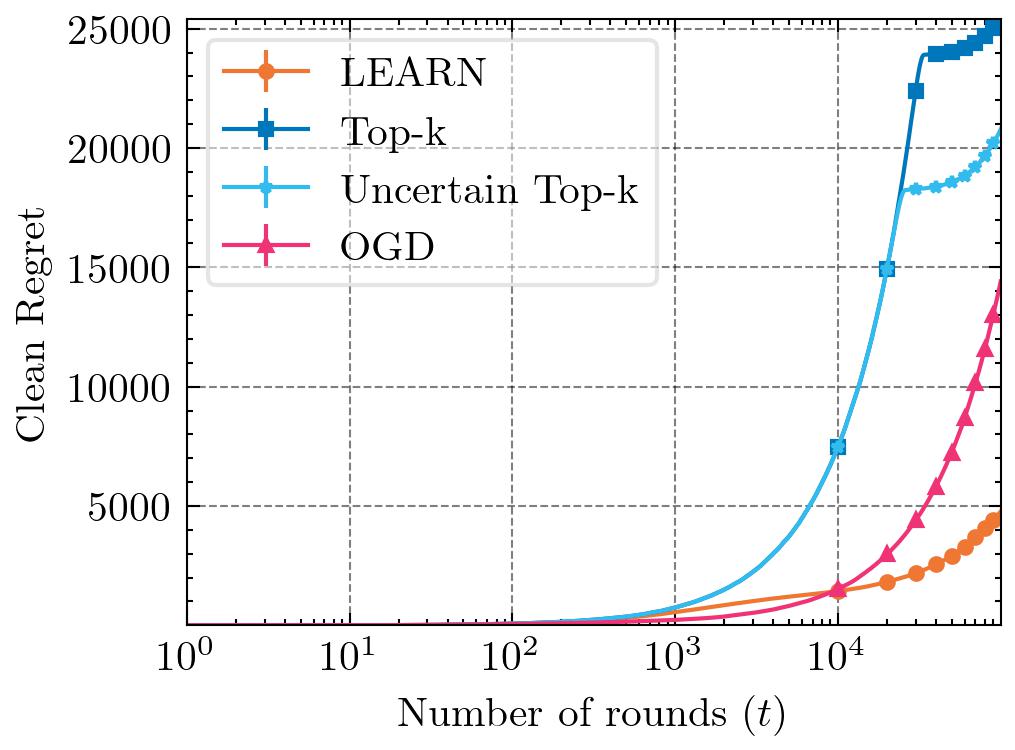}
        \caption{\label{fig: lin reg kcfr k constant} $k = \frac{T}{4}$ }
    \end{subfigure}
    \begin{subfigure}[t]{0.25\textwidth}
    \includegraphics[scale=0.4]{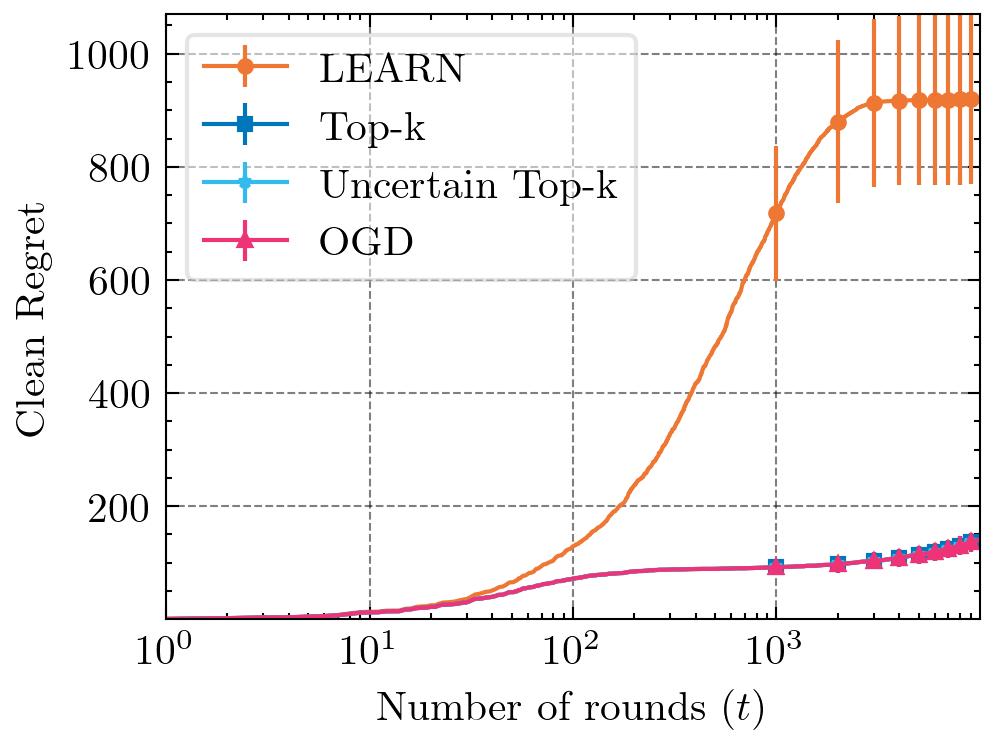}
        \caption{\label{fig: svm kcfr k0} $k = 0$ }
    \end{subfigure}%
    \begin{subfigure}[t]{0.25\textwidth}
    \includegraphics[scale=0.4]{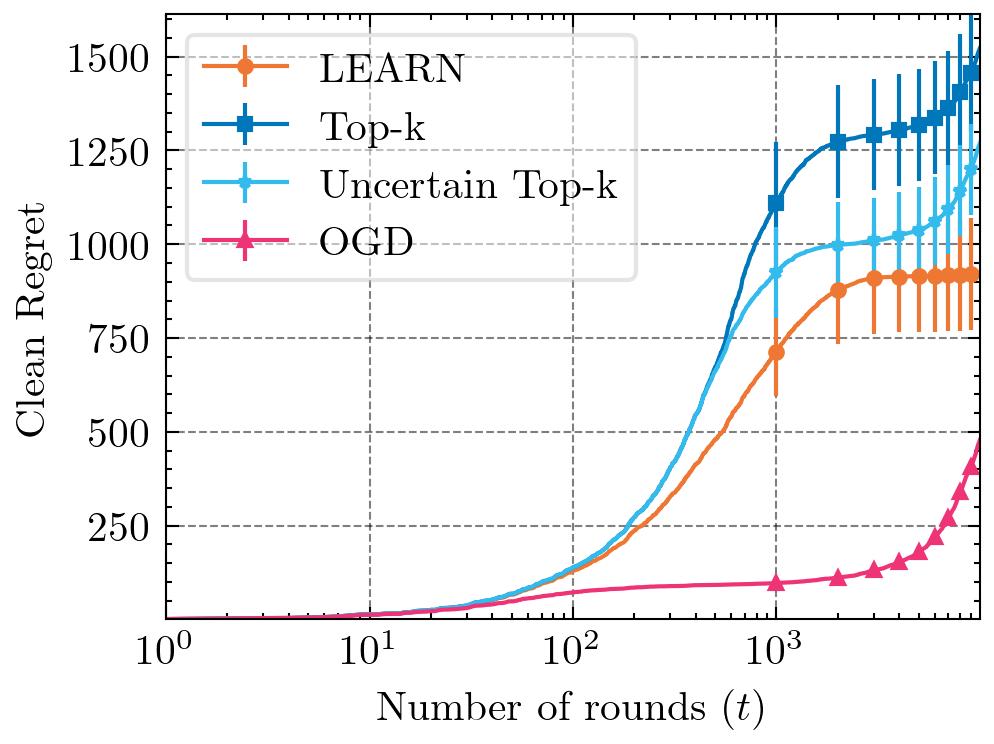}
        \caption{\label{fig: svm kcfr ksqrt} $k = \sqrt{T}$ }
    \end{subfigure}%
    \begin{subfigure}[t]{0.25\textwidth}
    \includegraphics[scale=0.4]{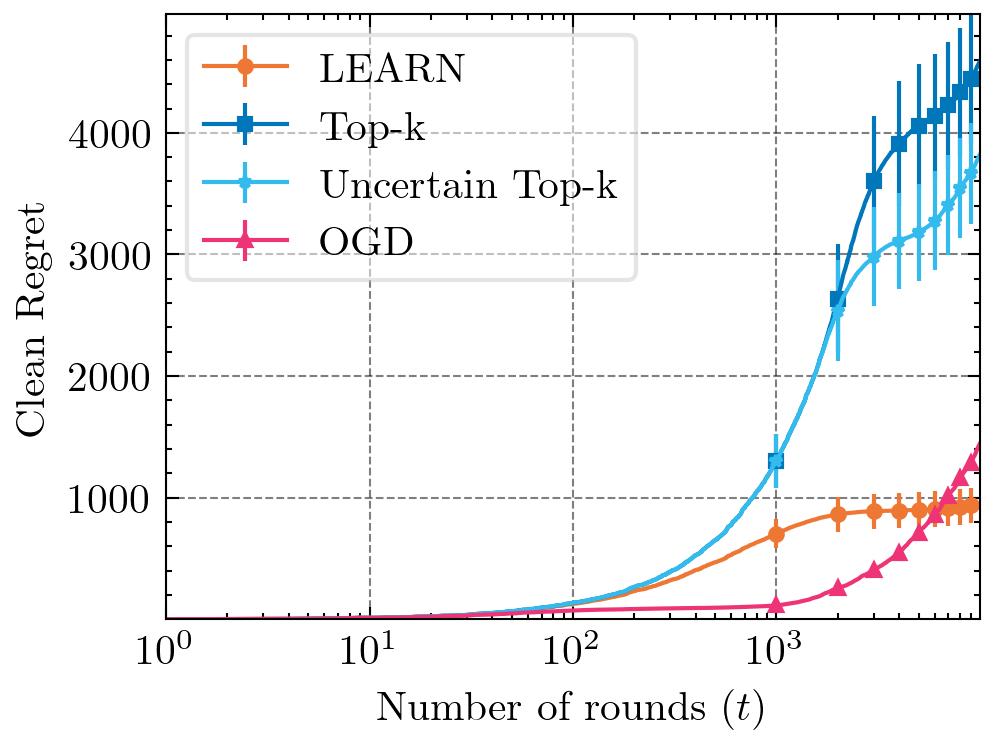}
        \caption{\label{fig: svm kcfr ktwothird} $k = T^{\frac{2}{3}}$ }
    \end{subfigure}%
    \begin{subfigure}[t]{0.25\textwidth}
    \includegraphics[scale=0.4]{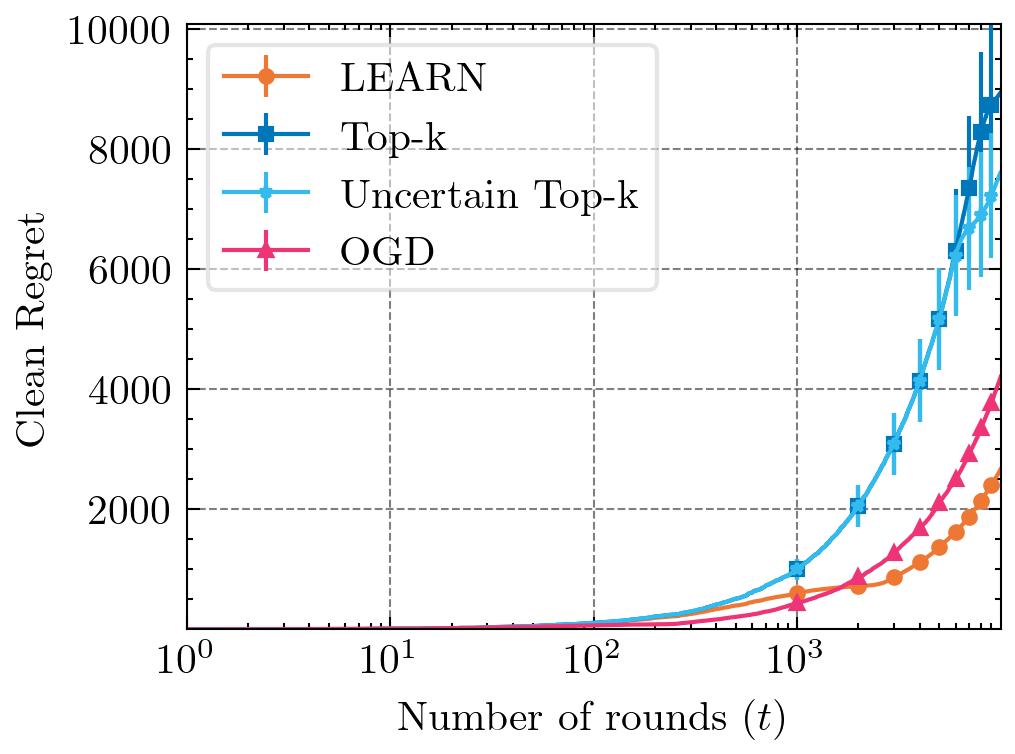}
        \caption{\label{fig: svm kcfr k constant} $k = \frac{T}{4}$ }
    \end{subfigure}\par
    \caption{\label{fig:  Clean regret for svm and lin reg} Clean dynamic regret plots for different algorithms running on the Online Ridge Regression (Top row, Figures~\ref{fig: lin reg kcfr k0}-\ref{fig: lin reg kcfr k constant}) and Online SVM (Bottom row, Figures~\ref{fig: svm kcfr k0}-\ref{fig: svm kcfr k constant}).}
\end{figure*}

\paragraph{Online linear regression}
In this experiment, the learner encounters the following loss function $f_t$ at each round for a fixed $\lambda = 10^{-4}$:
\begin{align}
\label{eq:online lin reg loss main}
f_t\big( (x_t, y_t), \theta \big) = \frac{\lambda}{2} \| \theta \|^2 +  (y_t - \inner{x_t}{\theta})^2~.
\end{align}
The response for uncorrupted rounds is generated using the linear model $y_t = \inner{\theta^*}{x_t} + e_t$ with a unit norm $\theta^* \in [-1, 1]^{100}$. The entries of $x_t$ are drawn independently from the normal distribution $\mathcal{N}(0, 1)$. Additionally, the additive noise $e_t$ is independently chosen from $\mathcal{N}(0, 10^{-6})$.  For corrupted rounds, $y_t$ is chosen uniformly at random from the interval $[0, 1]$. Figures~\ref{fig: lin reg kcfr k0}-\ref{fig: lin reg kcfr k constant} show the results from our experiments with varying numbers of outliers over $T=10^5$ rounds.

\paragraph{Online SVM}
Following \citet{shalev2007pegasos}, we address the primal version of the online SVM problem. With $\lambda$ set at $10^{-4}$, consistent with \citet{shalev2007pegasos}, the learner faces the following loss in each round $t$:
\begin{align}
\label{eq:online-svm-loss-main}
f_t\big( (x_t, y_t), \theta \big) = \frac{\lambda}{2} \| \theta \|^2 + \max(0, 1 - y_t \inner{x_t}{\theta})~.
\end{align}
The labels for the uncorrupted rounds are generated according to $y_t = \text{sign}\big( \inner{\theta^*}{x_t} \big)$, where $\theta^* \in [1, 11]^2$. Mislabeling occurs with  probability $0.05$  when $| \inner{\theta^*}{x_t} | \leq 0.1$. The entries of $x_t$ are drawn independently from $\mathcal{N}(0, 100)$. For corrupted rounds, the adversary flips the sign of $y_t$. Figures~\ref{fig: svm kcfr k0}-\ref{fig: svm kcfr k constant} show the experimental results with varying numbers of outliers over $T=10^4$ rounds.

In the absence of outliers ($k = 0$), the results for both online linear regression (Figure \ref{fig: lin reg kcfr k0}) and Online SVM (Figure~\ref{fig: svm kcfr k0}) illustrate a significant gap in the cumulative clean regret of \Learn{} compared to other baselines. \Learn{} is intentionally designed to anticipate outliers in the data, leading to cautious and slow updates initially, resulting in a substantial accumulation of regret in the beginning.
However, the regret plot for \Learn{} levels off after the initial few rounds.
As the number of outliers increase (Figures \ref{fig: lin reg kcfr ksqrt} and \ref{fig: lin reg kcfr ktwothird} for online linear regression and Figures~\ref{fig: svm kcfr ksqrt} and \ref{fig: svm kcfr ktwothird} for online SVM), the performance of all methods, except \Learn{}, deteriorates significantly. Although OGD initially incurs small regret, its growth rate is the largest, implying that it will eventually surpass other methods. \Learn{}, on the other hand, maintains roughly the same regret until $k = T^{\frac{2}{3}}$.
This demonstrates \Learn{}'s robustness to outliers,
even with higher outlier proportions. These observations corroborate our result in equation~\eqref{eq: clean bounded domain dynamic regret} of Theorem \ref{thm: bounded domain clean dynamic regret}, which predicts sublinear behavior for sublinear $k$. Lastly, recall that with $k = \frac{T}{4}$, equation~\eqref{eq: clean bounded domain dynamic regret} in Theorem~\ref{thm: bounded domain clean dynamic regret} does not assure sublinear regret. Correspondingly, \Learn{} and other methods exhibit increasing regrets in this regime.
\section{Conclusion and future work}
In this work, we established tight dynamic regret guarantees within the robust OCO framework. Introducing the \Learn{} algorithm, we showcased its effectiveness in handling outliers even when the outlier count $k$ remained unknown, particularly when the loss functions are $m$-strongly convex. Moreover, our algorithm worked with unbounded domains and accommodated large gradients. Central to our algorithm, we employ \Learn{} loss $g$, an invex robust loss designed to diminish the influence of outliers during gradient descent style updates. Additionally, we formulated a unified analysis framework to develop online optimization algorithms tailored for a class of non-convex losses referred to as invex losses. This framework serves as the basis for obtaining regret bounds associated with the \Learn{} loss. To substantiate our theoretical findings, we also conducted numerical experiments. A few natural extensions of our work include studying the robust loss $g_t$ with dynamic parameters $a_t$ and $b_t$ and leveraging the unified analysis framework to develop novel algorithms for various invex losses. In future, exploring the application of invex losses in online learning beyond the context of handling outliers would be interesting.


\bibliography{main}
\bibliographystyle{apalike}

\newpage
\appendix

\section{Related Work}
\label{sec:related work}
This section provides a detailed overview of closely related lines of research that inform our work.

\paragraph{Outlier Robust Online Learning} Our work builds on the work of ~\citet{van2021robust} where they consider the OCO framework when outliers are present. They specifically address the scenario where $k$, the number of outliers, is known in advance, and the adversary can arbitrarily corrupt any of the $k$ rounds. The model assumes that inliers have a bounded range (not known), while outliers can extend to an unbounded range. Notably, learner performance is quantified through clean static regret. The core of their algorithm involves the maintaining of top-$k$ gradients, triggering updates to the online learning algorithm only when the norm of the gradient is less than 2 times the minimum norm within the top-$k$ gradients. They establish robust regret guarantees, offering $\calO(G(\calS)(\sqrt{T} + k))$  for convex functions and $\calO(\log T + G(\calS)k)$ for strongly convex functions, where $G(\calS)$ represents the norm of the largest gradient considering only clean rounds. Their work, while providing valuable insights, assumes known values for $k$, a bounded domain, and the existence of a Lipschitz bound $G(\calS)$ for clean rounds. In contrast, our work extends beyond these assumptions, particularly focusing on dynamic environments where the number of outliers is unknown, the domain is unbounded, and the losses may not be Lipschitz.

\paragraph{Robust Losses}
 A pivotal element of our investigation lies in formulating and examining a non-convex robust loss function. To contextualize our work, we highlight pertinent robust losses featured in the existing literature. Notable examples include Huber loss~\citep{huber1964robust}, Tukey's biweight loss~\citep{tukey1974fitting}, Welsch loss~\citep{dennis1978techniques} and Cauchy loss~\citep{black1996robust} among others, each tailored to address different applications. Huber loss, for instance, strikes a balance between mean squared error and mean absolute error, rendering it suitable for robust regression~\citep{huber1973robust}. Similarly, Tukey's biweight loss, known for its efficacy in robust regression, provides resistance against outliers through a truncated quadratic penalty~\citep{chang2018robust}. It is worth noting that Tukey's biweight loss, along with Cauchy and Welsch losses, are non-convex and satisfy the \emph{redescending} property~\citep{hampel2011robust}. \Learn{} loss also shares this attribute. For a detailed comparison of \Learn{} loss with different robust losses see Appendix~\ref{sec: comparison of robust losses}. Various generalizations and extensions of these losses have been studied ~\citep{barron2019general,gokcesu2021generalized}, reflecting the extensive exploration and adaptation of robust loss functions within the broader literature.

\paragraph{OCO with unbounded domains and gradients}
It was only in the recent work of~\citet{jacobsen2023unconstrained} that a sublinear regret guarantee was achieved by avoiding the assumption that there exists $\|\nabla f_t\|\leq G$ for all rounds $t$. They obtain sublinear regret guarantees in the standard OCO framework when the domain is unbounded, and the gradients grow large with the norm of actions ($\| \theta_t \|$).
They provide the first algorithm for dynamic regret minimization in the unbounded domain and gradient setup. In this work, we adopt an expert scheme akin to theirs to handle unbounded domains and gradients that can grow large in the robust outlier framework.

\paragraph{Related Applications} The work of~\citet{chen2022online} studies online linear regression within a realizable context, accounting for small noise using the Huber contamination model~\citep{huber1964robust}. However, our model (similarly the model of~\citet{van2021robust}) differs from theirs significantly. Notably, we do not impose any assumptions of realizability or control the corruption mechanism through probabilistic assumptions. Although distinct in the techniques employed, it is important to acknowledge a substantial body of related work addressing robust optimization in various theoretical computer science and machine learning problems. For a comprehensive overview, the textbook authored by \citet{diakonikolas2023algorithmic} stands out as an excellent reference.

\section{Dynamic Regret in a Bounded Domain Using Top-k Filter Algorithm}
\label{sec: extension of van erven's result}

\citet{van2021robust} established a tight bound on static regret within a bounded domain, assuming prior knowledge of $k$. For convex losses, they achieved an optimal static regret upper bound of $\mathcal{O}(\sqrt{T} + k)$. In the case of strongly convex losses, this bound is further improved to $\mathcal{O}(\log T + k)$, which remains optimal. In this section, we extend their results to the dynamic setting and establish a clean dynamic regret bound of $\mathcal{O}(\sqrt{V_T T} + k)$ within a bounded domain. Notably, this bound remains tight even for strongly convex functions~\citep{zhang2017improved}.

We make the following assumptions for our extension:
\begin{enumerate}
    \item The Lipschitz-adaptive algorithm ALG used by~\cite{van2021robust} is online gradient descent (OGD).
    \item The domain $\Theta$ is bounded.
\end{enumerate}

Let $\mathcal{S} \subseteq [T]$ be the set of uncorrupted rounds. Then, we are interested in bounding the clean dynamic regret defined as:

\begin{align}
\mathfrak{R}_{RD}^T(\mathbf{\theta}, \mathcal{S}) = \sum_{t \in \mathcal{S}} \Big( f_t(s_t, \theta_t) - f_t(s_t, \theta_t^*) \Big)~,
\end{align}

Let $\mathcal{F} \subset [T]$ denote the rounds flitered out by Algorithm 1 in \cite{van2021robust}, and let $\mathcal{P} = [T] \setminus \mathcal{F}$ denote the rounds that are passed on to ALG. Then,

\begin{align}
\mathfrak{R}_{RD}^T(\mathbf{\theta}, \mathcal{S}) = \underbrace{\sum_{t \in \mathcal{S} \cap \mathcal{P}} \Big( f_t(s_t, \theta_t) - f_t(s_t, \theta_t^*) \Big)}_{T_1} + \underbrace{\sum_{t \in \mathcal{S} \cap \mathcal{F}} \Big( f_t(s_t, \theta_t) - f_t(s_t, \theta_t^*) \Big)}_{T_2}
\end{align}

First, we bound $T_1$.

\begin{align}
T_1 = \underbrace{\sum_{t \in \mathcal{P}} \Big( f_t(s_t, \theta_t) - f_t(s_t, \theta_t^* ) \Big)}_{T_{11}} - \underbrace{\sum_{t \in  \mathcal{P} \setminus \mathcal{S}} \Big( f_t(s_t, \theta_t) - f_t(s_t, \theta_t^*) \Big)}_{T_{12}}
\end{align}

The bound for the term $T_{12}$ follows directly from \cite{van2021robust}'s proof of their Theorem 1:
\begin{align}
T_{12} \leq \sum_{t \in  \mathcal{P} \setminus \mathcal{S}} \nabla f_t(s_t, \theta_t)^T(\theta_t - \theta_t^*) \leq 2DG(\mathcal{S})k~,
\end{align}
where $D = \max_{a, b \in \Theta} \| a - b \|$ and $G(\mathcal{S}) = \max_{t \in \mathcal{S}} \| \nabla f_t(s_t, \theta_t) \|$.

Similarly, we take the bound for $T_2$ directly from \cite{van2021robust}'s proof of Theorem 1:
\begin{align}
T_2 \leq  \sum_{t \in \mathcal{S} \cap \mathcal{F}} \nabla f_t(s_t, \theta_t)^T (\theta_t - \theta_t^*) \leq 2 D G(\mathcal{S}) (k+1)
\end{align}
At this point, it only remains to bound $T_{11}$. Recall that OGD (instantiation of ALG) only sees the rounds in $\mathcal{P}$, and by reindexing them as $t \in \{1, \ldots, |\mathcal{P}|\}$, we can write for all $t \in [|\mathcal{P}|]$:
\begin{align}
\| \theta_{t+1} - \theta_{t+1}^* \|^2 = \| \theta_{t+1} - \theta_t^* + \theta_t^* - \theta_{t+1}^* \|^2
\end{align}

Expanding RHS and by definition of $D$, we have:
\begin{align}
\| \theta_{t+1} - \theta_{t+1}^* \|^2 \leq \| \theta_{t+1} - \theta_t^* \|^2 + 3 D \| \theta_t^* - \theta_{t+1}^* \|
\end{align}

Now, substituting $\theta_{t+1} = \Pi_{\theta \in \Theta} \Big( \theta_t - \alpha_t \nabla f_t(s_t, \theta_t) \Big)$ and using the contractive property of the projection on a convex set, we can write:
\begin{align}
\| \theta_{t+1} - \theta_{t+1}^* \|^2 \leq \| \theta_{t} - \alpha_t \nabla f_t(s_t, \theta_t)  - \theta_t^* \|^2 + 3 D \| \theta_t^* - \theta_{t+1}^* \|
\end{align}
Simplifying,
\begin{align}
\| \theta_{t+1} - \theta_{t+1}^* \|^2 \leq \| \theta_{t}  - \theta_t^* \|^2 + \alpha_t^2 \| \nabla f_t(s_t, \theta_t) \|^2 - 2 \alpha_t \nabla f_t(s_t, \theta_t)^T (\theta_t - \theta_t^*) + 3 D \| \theta_t^* - \theta_{t+1}^* \|
\end{align}

Taking $\alpha_t = \alpha > 0$ and using the convexity of $f_t(s_t, \theta_t)$, we get:
\begin{align}
2 \alpha \Big( f_t(s_t, \theta_t) - f_t(s_t, \theta_t^*) \Big) \leq  \| \theta_{t}  - \theta_t^* \|^2 - \| \theta_{t+1} - \theta_{t+1}^* \|^2  + \alpha^2 \| \nabla f_t(s_t, \theta_t) \|^2 + 3 D \| \theta_t^* - \theta_{t+1}^* \|
\end{align}

Summing it across all $t \in [|\mathcal{P}|]$:
\begin{align}
T_{11} \leq \frac{\| \theta_1 - \theta_1^*\|^2}{2\alpha} + \frac{\alpha}{2}\sum_{t \in \mathcal{P}}\| \nabla f_t(s_t, \theta_t) \|^2 + \frac{3D}{2\alpha} \sum_{t \in \mathcal{P}} \| \theta_t^* - \theta_{t+1}^* \|
\end{align}

We make two observations:
\begin{enumerate}
    \item $\| \nabla f_t(s_t, \theta_t) \| \leq 2G(\mathcal{S}), \forall t \in \mathcal{P}$ - due to \cite{van2021robust}
    \item $\sum_{t \in \mathcal{P}} \| \theta_t^* - \theta_{t+1}^* \| \leq \sum_{t \in [T]} \| \theta_t^* - \theta_{t+1}^*\| = V_T$
\end{enumerate}
We choose $\alpha =\frac{\sqrt{3DV_T + D^2}}{2G(\mathcal{S})\sqrt{T}}$ to get:
\begin{align}
T_{11} \leq 2 \sqrt{(3D V_T + D^2) T} G(\mathcal{S})~.
\end{align}
Combining all the above results together, we get
\begin{align}
\mathfrak{R}_{RD}^T(\mathbf{\theta}, \mathcal{S}) \leq  2 \sqrt{(3D V_T + D^2) T} G(\mathcal{S}) + 2DG(\mathcal{S})k + 2 D G(\mathcal{S}) (k+1) = \mathcal{O}( \sqrt{V_T T} + k )~.
\end{align}

\section{Comparison of Robust Losses}
\label{sec: comparison of robust losses}
\begin{figure}[th!]
  \centering
  \includegraphics[width=0.5\columnwidth]{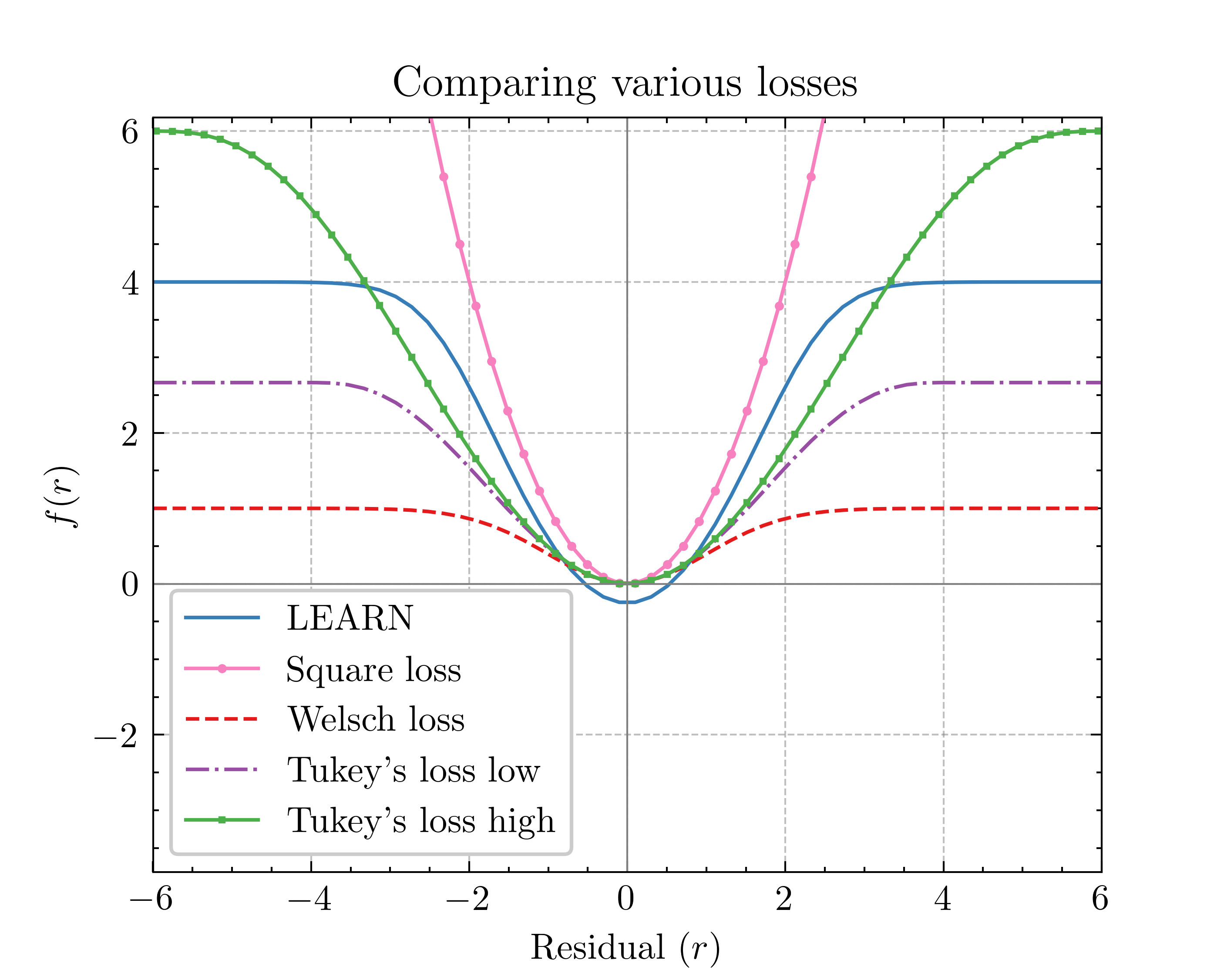}
  \caption{The robust loss closely follows the square loss and flattens out as we move away from the minima.}
  \label{fig: loss comparison}
\end{figure}

To understand how the robust loss works to mitigate its susceptibility to outliers, we examine the behavior of different non-convex losses, including our robust loss. We compare how these losses resemble the square loss $f$ near the origin and observe their behavior as we deviate from it.

Let us consider the square loss function $f(r)=r^2$ where $r$ denotes the residual, and the minimum is attained at  $0$. Our robust loss, denoted by $g$, instantiated with $f(r)$ with constants $a=2$ and $b=e^{-2}$, is represented in the Figure~\ref{fig: loss comparison}.
Note that $g$ closely follows the behavior of the function $f$ in the vicinity of the origin and gradually levels off as we move away from it. Importantly, the minimizers for both $g$ and $f$ are the same.

In Figure~\ref{fig: loss comparison}, we plot two other non-convex losses that are similar in appearance to our robust loss. First, we examine Tukey's biweight loss, defined as
\begin{align}
\ell(r)=\left\{\begin{array}{ll}\frac{c^{2}}{6}\left(1-\left[1-\left(\frac{r}{c}\right)^{2}\right]^{3}\right) & \text { if }|r| \leq c \\ \frac{c^{2}}{6} & \text { otherwise }\end{array}\right.
\end{align}
where $r$ denotes the residual and $c$ is a tunable parameter. We observe that, even when tuning the parameter $c$ to both high$(c=6)$ and low$(c=4)$ values, the curve fails to closely follow the curvature of the square loss, contrasting with the behavior of our robust loss. Moving away from the origin, Tukey's loss flattens out at a slower when compared to $g$.

Next, we consider the Welsch loss defined as
\begin{align}
    h(r)=1-\exp\left(-\frac{1}{2}\left(\frac{r}{c}\right)^2\right)
\end{align}
Here, the tunable parameter is set to $c=2$ in Figure~\ref{fig: loss comparison}. However, its limitation of being capped at one impedes its ability to track the underlying square loss precisely. This causes it to flatten out early, making it less suitable for online learning. These observations shed light on the distinctive characteristics of these losses and underscore the efficacy of our proposed robust loss in capturing the desired behavior.

\section{Invexity of the Robust Loss}
\label{sec: invexity of the robust loss}

\subsection{Proof of Lemma~\ref{lem: g_t is invex}}
\label{subsec: proof of lemma g_t invex}

\begin{proof}
\label{proof: g_t is invex}
To prove that $g_t$ is a $\zeta_t$-invex function, it suffices to show that all its stationary points are the global minima. Using the monocity of $\log(\cdot)$ and $\exp(\cdot)$ functions, it follows that
\begin{align}
    \overline{\theta}_t \in \arg\min_{\theta \in \real^d} f_t(s_t, \theta) = \arg\min_{\theta \in \real^d} g_t(s_t, \theta)
\end{align}

Due to the convexity of $f_t$, either $ \nabla f_t(\overline{\theta}) = 0 $ or $f_t$ does not have any stationary point. Now,
\begin{align}
    \label{eq: grad of gt for invexity}
    \nabla g_t(\theta) = \frac{\exp(-\frac{1}{a} f_t(s_t, \theta))}{b + \exp(-\frac{1}{a} f_t(s_t, \theta))} \nabla f_t(\theta)
\end{align}

It is obvious that $\nabla g_t(\theta) = 0$ if and only if $\nabla f_t(\theta) = 0$. This immediately implies that $g_t$ is invex. Equivalently, there exists a function $\zeta_t(\vartheta_2, \vartheta_1)$ such that
\begin{align}
        g_t(s_t, \vartheta_2) \geq g_t(s_t, \vartheta_1) + \inner{\nabla g_t(s_t, \vartheta_1)}{\zeta_t(\vartheta_2, \vartheta_1)}~.
\end{align}

Again, employing the monocity of $\log(\cdot)$ and $\exp(\cdot)$ functions,
\begin{align}
    \omega_t^* \coloneqq \arg\min_{\theta \in \Theta} f_t(s_t, \theta) = \arg\min_{\theta \in \Theta} g_t(s_t, \theta)
\end{align}

Let
\begin{align*}
    \eta_t = \frac{\exp(-\frac{1}{a} f_t(s_t, \theta_t))}{b + \exp(-\frac{1}{a} f_t(s_t, \theta_t))}, \; \eta_t^* = \frac{\exp(-\frac{1}{a} f_t(s_t, \omega_t^*))}{b + \exp(-\frac{1}{a} f_t(s_t, \omega_t^*))}
\end{align*}

Observe that $\eta_t^* \geq \eta_t$ as a consequence of $ f_t(s_t, \omega_t^*) \leq f_t(s_t, \theta_t) $. Next, we analyze the following quantity:
\begin{align}
    g_t(s_t, \theta_t) - g_t(s_t, \omega_t^*) &= - a \log (\exp( -\frac{1}{a} f_t(s_t, \theta_t) ) + b) + a \log (\exp( -\frac{1}{a} f_t(s_t, \omega_t^*) ) + b) \\
    &= a \log \left( \frac{\exp( -\frac{1}{a} f_t(s_t, \omega_t^*) ) + b}{\exp( -\frac{1}{a} f_t(s_t, \theta_t) ) + b} \right)\\
    &= a \log \left( \frac{\eta_t}{\eta_t^*} \frac{\exp( -\frac{1}{a} f_t(s_t, \omega_t^*) )}{\exp( -\frac{1}{a} f_t(s_t, \theta_t) )} \right) \\
    &= a \log \frac{\eta_t}{\eta_t^*} + f_t(s_t, \theta_t) - f_t(s_t, \omega_t^*) \\
    \label{eq:eta_t<eta_t*}&\leq f_t(s_t, \theta_t) - f_t(s_t, \omega_t^*) \\
    \label{eq:convexity}&\leq \inner{\nabla f_t(s_t, \theta_t)}{\theta_t - \omega_t^*}\\
    \label{eq:nabla g_t}&= \frac{1}{\eta_t} \inner{\nabla g_t(s_t, \theta_t)}{\theta_t - \omega_t^*}~.
\end{align}

Note that the inequality~\eqref{eq:eta_t<eta_t*} is due to $\eta_t \leq \eta_t^*$, inequality~\eqref{eq:convexity} is due to convexity of $f_t$ and Equation~\eqref{eq:nabla g_t} is by observing that $\nabla g_t(s_t, \theta_t) = \eta_t \nabla f_t(s_t, \theta_t)$. This completes the proof.

\paragraph{A note on the nondifferentiability of $f_t$} In the scenario of a nondifferentiable yet convex function $f_t$, the first-order optimality condition implies $0 \in \partial f_t(\theta)$, where $\partial f_t(\cdot)$ denotes the subdifferential set. It's worth noting that the assertion immediately following equation~\eqref{eq: grad of gt for invexity} remains valid, taking into account the condition $0 \in \partial f_t(\overline{\theta}_t)$.
\end{proof}

\subsection{Proof of Theorem~\ref{thm: bound on invex regret}}
\label{subsec: proof of thm: bound on invex regret}

 \begin{proof}
     \label{proof: bound on invex regret}
    We establish that our problem's setting satisfies the necessary conditions to apply the unified analysis framework detailed in Appendix~\ref{sec: dynamic regret for invex losses}. Using the Euclidean distance as our distance function $\Delta$, it is straightforward to verify that Assumption~\ref{assum: triangle inequality} holds, with $\gamma_1 = 1$ and $\gamma_2 = 1$. The assumption of a bounded domain, Assumption~\ref{assum: bounded domain}, remains valid when substituting $R$ with $2D$.

     Next, we verify Assumption~\ref{assum: first-order update property}.
     \begin{align}
          \| \theta_{t+1} - \theta_t^* \|^2
    &= \| \Pi_{\Theta}(\theta_t - \alpha_t \nabla g_t(s_t, \theta_t)) - \theta_t^* \|^2 \\
    \label{eq: projection on convex set}&\leq \| \theta_t - \alpha_t \nabla g_t(s_t, \theta_t) - \theta_t^* \|^2 \\
    &= \| \theta_t - \theta_t^* \|^2 + \alpha_t^2  \| \nabla g_t(s_t, \theta_t) \|^2 - 2 \alpha_t  \inner{- \nabla g_t(s_t, \theta_t)}{\theta_t^* -  \theta_t} \\
    &= \| \theta_t - \theta_t^* \|^2 + \alpha_t^2  \| \nabla g_t(s_t, \theta_t) \|^2 - 2 \eta_t \alpha_t  \inner{- \nabla g_t(s_t, \theta_t)}{\frac{\theta_t^* -  \theta_t}{\eta_t}}~,
     \end{align}
     where inequality~\eqref{eq: projection on convex set} follows from the contraction of the projection on a convex set and $\eta_t$ is defined in the proof of Lemma~\ref{lem: g_t is invex}. Using the results from Lemma~\ref{lem: g_t is invex}, we can verify that the Assumption~\ref{assum: first-order update property} holds with $\gamma_3^t = 1$ and $\gamma_4^t = \frac{1}{\eta_t}$. Now, observe that $\max_{t =1}^T \gamma_4^t \leq (1 + b \exp(\frac{B}{a}))$. By choosing a constant step size $\alpha_t = \alpha$ for all $t \in [T]$ and using the results from Theorem~\ref{thm: bound on invex dynamic regret}, we obtain the following regret bound:
     \begin{align}
        \label{eq: bound on invex dynamic regret g_t}
        \regret_{\id}^T(\btheta) &\leq \left(1 + b \exp\left(\frac{B}{a}\right)\right) \left(  \frac{ 4D^2 }{ 2\alpha } + \frac{\alpha}{2} \sum_{t=1}^T \| \nabla g_t(\theta_t) \|^2 +\frac{ 6D \sum_{t=1}^T \| \theta_t^* - \theta_{t+1}^* \|}{2 \alpha} \right)~.
    \end{align}

   We observe that $\nabla g_t(\theta_t) = \eta_t \nabla f_t(\theta_t)$ and then apply Lemma~\ref{lem: bound on eta grad^r} to write:
   \begin{align}
   \label{eq: bound on invex dynamic regret g_t 1}
       \regret_{\id}^T(\btheta) &\leq \left(1 + b \exp\left(\frac{B}{a}\right)\right)\left(  \frac{ 4D^2 }{ 2\alpha } + \frac{\alpha}{2}  \conC^2 T +\frac{ 6D V_T}{2 \alpha} \right)~,
   \end{align}
   where $C_{m, G, L}(a, b)$ is defined in Lemma~\ref{lem: bound on eta grad^r} and optimal path variation $V_T \coloneqq \sum_{t=1}^T \| \theta_t^* - \theta_{t+1}^* \|$. By choosing,
   \begin{align}
       \alpha = \sqrt{ \frac{ 4D^2 + 6DV_T }{ \conC^2 T } }~,
   \end{align}
   we get
   \begin{align}
       \label{eq: bound on invex dynamic regret g_t 2}
       \regret_{\id}^T(\btheta) &\leq \left(1 + b \exp\left(\frac{B}{a}\right)\right) \conC \sqrt{ ( 4D^2 + 6DV_T)T }~.
   \end{align}

 \end{proof}

\section{A Unified Framework to Bound Dynamic Regret for Online Invex Optimization}
\label{sec: dynamic regret for invex losses}

In this section, we develop a unified framework to bound dynamic regret for invex losses using a first-order algorithm. First, we present some definitions and assumptions which will be used later to construct our framework.

\begin{definition}[Invex functions]
    A differentiable function $g: \real^d \to \real$ is called a $\zeta$-invex function on set $\Theta$ if for any $\vartheta_1, \vartheta_2 \in \Theta$ and a vector-valued function $\zeta: \Theta \times \Theta \to \real^d$,
    \begin{align}
        \label{eq: invex function}
        g(\vartheta_2) \geq g(\vartheta_1) + \inner{\nabla g(\vartheta_1)}{\zeta(\vartheta_2, \vartheta_1)}
    \end{align}
    where $\nabla g(\vartheta_1)$ is the gradient of $g$ at $\vartheta_1$.
\end{definition}

\begin{remark}
    Conventionally, the notation $\eta(\cdot, \cdot)$ is employed instead of $\zeta(\cdot, \cdot)$ within the definition of an invex function. The preference for $\zeta$ in our work arises from the prior use of $\eta$ for a distinct purpose elsewhere in our paper.
\end{remark}
\begin{remark}
    When working in a non-Euclidean Riemannian manifold the inner product $\inner{\cdot}{\cdot}$ will be defined in the tangent space of $\theta_1$ induced by the Riemannian metric.
\end{remark}
\begin{remark}
    The choice of $\zeta(\cdot, \cdot)$ may not be unique.
\end{remark}

\begin{definition}[First-order Algorithm]
\label{def: first-order algorithm}
We denote an algorithm as $\calA \coloneqq \calA(\theta_t, \nabla g(\theta_t), \alpha_t, \zeta)$, where, at round $t$, it accepts the current iterate $\theta_t \in \Theta$, gradient $\nabla g(\theta_t) \in \real^d$, stepsize $\alpha > 0$, and the function $\zeta$ as inputs, and produces the next iterate $\theta_{t+1} \in \Theta$ as its output.
\end{definition}

\begin{definition}[Distance function]
\label{def: distance function}
The function $\Delta: \real^d \times \real^d \to \real$ is called a distance function.\footnote{This concept can be extended to encompass non-Euclidean manifolds. Specifically, one can engage with Hadamard manifolds, taking into account geodesically convex functions.}
\end{definition}
\begin{remark}
   Although $\Delta$ can represent a metric associated with a metric space, we do not explicitly impose this requirement.
\end{remark}

We consider an online invex optimization (OIO) framework akin to OCO where the learner engages in a sequential decision-making process. Specifically, in each round $t \in [T]$, the learner selects an action $\theta_t \in \Theta$ and observes a $\zeta_t$-invex loss denoted as $g_t:\Theta\to \real$. Subsequently, the learner employs a first-order optimization algorithm, denoted as $\mathcal{A}$, to perform the next action $\theta_{t+1} \in \Theta$. All the procedural details of this algorithm are presented in the following section.

\begin{algorithm}[H]
    \caption{First-order algorithm for OIO}\label{alg:online invex optimization}
    \begin{algorithmic}
        \STATE \textbf{Initialize: } $\theta_1 \in \Theta$
        \FOR{$t = 1, \ldots, T$}
        \STATE \textbf{Play: } $\theta_t$
        \STATE \textbf{Observe: } $\zeta_t$-invex loss $g_t$
        \STATE \textbf{Update: } $\theta_{t+1} \gets \calA(\theta_t, \nabla g_t(\theta_t), \alpha_t, \zeta_t)$
        \ENDFOR
    \end{algorithmic}
\end{algorithm}

The learner aims to minimize the dynamic regret as defined below:
\begin{align}
    \regret_{\id}^T(\btheta) \coloneqq \sum_{i=1}^T g_t(\theta_t) - \sum_{i=1}^T g_t(\theta_t^*) \; ,
\end{align}
where $\theta_t^* = \arg\min_{\theta \in \Theta} g_t(\theta)$ and $\btheta = (\theta_1, \ldots, \theta_T)$.

In this section, we use the variable $\gamma$ with subscripts and superscripts to signify positive and finite ($0\leq \Gamma < \infty$) parameters. With this notational clarification, we are prepared to enumerate the assumptions essential for our setup.

\begin{assumption}[Generalized Law of Cosines]
    \label{assum: triangle inequality}
    For all $\vartheta_1, \vartheta_2, \vartheta_3 \in \real^d$, the distance function satisfies the following inequality:
    \begin{align}
        \label{eq: triangle inequality}
        \Delta(\vartheta_2, \vartheta_1)^2 \leq \Delta(\vartheta_2, \vartheta_3)^2 + \gamma_1 \Delta(\vartheta_3, \vartheta_1)^2 + 2 \gamma_2 \Delta(\vartheta_2, \vartheta_3) \Delta(\vartheta_3, \vartheta_1) \; .
    \end{align}
\end{assumption}

\begin{assumption}[Bounded domain]
    \label{assum: bounded domain}
    For all $\vartheta_1, \vartheta_2 \in \Theta$,
    \begin{align}
        \label{eq: bounded domain}
        \Delta(\vartheta_1, \vartheta_2) \leq R \; .
    \end{align}
\end{assumption}

\begin{assumption}[First-order update property]
    \label{assum: first-order update property}
    Let $\theta_{t+1} \coloneqq \calA(\theta_t, \nabla g_t, \alpha_t, \zeta_t)$ be the chosen action at round $t+1$ and $\theta_t^* = \arg\min_{\theta \in \Theta} g_t(\theta)$, then
    \begin{align}
        \label{eq: first-order update property}
        \Delta(\theta_{t+1}, \theta_t^*)^2 \leq \Delta(\theta_t, \theta_t^*)^2 + \gamma_3^t \alpha_t^2 \| \nabla g_t(\theta_t) \|^2 - 2 \frac{1}{\gamma_4^t} \alpha_t \inner{- \nabla g_t(\theta_t)}{\zeta_t(\theta_t^*, \theta_t)} \; .
    \end{align}
\end{assumption}

Now, we are ready to state our theoretical result.

\begin{theorem}
    \label{thm: bound on invex dynamic regret}
    Let $g_t:\Theta \to \real$ for all $t \in [T]$ be a sequence of $\zeta_t$-invex losses and $\btheta = ( \theta_1, \ldots, \theta_T)$ be a sequence of actions generated by a first-order algorithm $\calA$. If there exists a distance function $\Delta: \real^d \times \real^d \to \real$ such that Assumptions~\ref{assum: triangle inequality},~\ref{assum: bounded domain} and~\ref{assum: first-order update property} are satisfied for a sequence of stepsizes $( \alpha_1, \ldots, \alpha_T)$, then the following holds:
    \begin{align}
        \label{eq: bound on invex dynamic regret}
        \begin{split}
        \regret_{\id}^T(\btheta) &\leq \sum_{t=1}^T \frac{ \gamma_4^t(\Delta(\theta_t, \theta_t^*)^2 - \Delta(\theta_{t+1}, \theta_{t+1}^*)^2  )  }{ 2\alpha_t } + \frac{1}{2} \sum_{t=1}^T \alpha_t \gamma_4^t  \gamma_3^t \| \nabla g_t(\theta_t) \|^2 \\
        &\quad +\frac{ R \sum_{t=1}^T \gamma_4^t (\gamma_1^t + 2 \gamma_2^t) \Delta(\theta_t^*, \theta_{t+1}^*)}{2 \alpha_t} \; .  \end{split}
    \end{align}
\end{theorem}
\begin{proof}
    \label{proof: bound on invex dynamic regret}
    By substituting, $\vartheta_1 = \theta_{t+1}^*, \vartheta_2 = \theta_{t+1}$ and $\vartheta_3 = \theta_t^*$ in Assumption~\ref{assum: triangle inequality}, we get:
    \begin{align}
        \Delta(\theta_{t+1}, \theta_{t+1}^*)^2 &\leq \Delta(\theta_{t+1}, \theta_t^*)^2 + \gamma_1^t \Delta(\theta_t^*, \theta_{t+1}^*)^2 + 2 \gamma_2^t \Delta(\theta_{t+1}, \theta_{t+1}^*) \Delta(\theta_t^*, \theta_{t+1}^*) \\
        \label{eq: bounded domain 1}&\leq \Delta(\theta_{t+1}, \theta_t^*)^2 + R(\gamma_1^t + 2 \gamma_2^t)  \Delta(\theta_t^*, \theta_{t+1}^*)
    \end{align}
    where inequality~\eqref{eq: bounded domain 1} follows from Assumption~\ref{assum: bounded domain}. Next, we bound $\Delta(\theta_{t+1}, \theta_t^*)^2$ using Assumption~\ref{assum: first-order update property} and get:
    \begin{align}
        \Delta(\theta_{t+1}, \theta_{t+1}^*)^2 &\leq \Delta(\theta_t, \theta_t^*)^2 + \gamma_3^t \alpha_t^2 \| \nabla g_t(\theta_t) \|^2 - 2 \frac{1}{\gamma_4^t} \alpha_t \inner{- \nabla g_t(\theta_t)}{\zeta_t(\theta_t^*, \theta_t)}  \\
        &\quad+ R(\gamma_1^t + 2 \gamma_2^t)  \Delta(\theta_t^*, \theta_{t+1}^*)
    \end{align}

    We get the following after rearranging the terms:
    \begin{align}
    \label{eq: rearrange}
        \inner{- \nabla g_t(\theta_t)}{\zeta_t(\theta_t^*, \theta_t)} &\leq \frac{\gamma_4^t(\Delta(\theta_t, \theta_t^*)^2 - \Delta(\theta_{t+1}, \theta_{t+1}^*)^2)}{2 \alpha_t} + \frac{\gamma_4^t \gamma_3^t \alpha_t \| \nabla g_t(\theta_t) \|^2}{2} \notag\\
        &\quad+ \frac{R \gamma_4^t (\gamma_1^t + 2 \gamma_2^t)  \Delta(\theta_t^*, \theta_{t+1}^*)}{2 \alpha_t}
    \end{align}

    Recall that due to $\zeta_t$-invexity of $g_t$, we have
    \begin{align}
        \label{eq: invexity of g_t}
        g_t(\theta_t) - g_t(\theta_t^*) \leq \inner{- \nabla g_t(\theta_t)}{\zeta_t(\theta_t^*, \theta_t)}
    \end{align}

    Combining inequalities~\eqref{eq: rearrange} and \eqref{eq: invexity of g_t}, we get:
    \begin{align}
    \label{eq: combine}
         g_t(\theta_t) - g_t(\theta_t^*) &\leq \frac{\gamma_4^t(\Delta(\theta_t, \theta_t^*)^2 - \Delta(\theta_{t+1}, \theta_{t+1}^*)^2)}{2 \alpha_t} + \frac{\gamma_4^t \gamma_3^t \alpha_t \| \nabla g_t(\theta_t) \|^2}{2} \notag \\
         &\quad + \frac{R \gamma_4^t (\gamma_1^t + 2 \gamma_2^t)  \Delta(\theta_t^*, \theta_{t+1}^*)}{2 \alpha_t}
    \end{align}
    We obtain the desired result by summing inequality~\eqref{eq: combine} from $t=1$ to $T$.
\end{proof}

The corollary below shows an explicit bound on $\regret_{\id}^T(\btheta)$ by appropriately choosing $\alpha_t$.

\begin{corollary}
    \label{cor: constant alpha}
    Let $\gamma_4^{\max} \coloneqq \max_{t=1}^T \gamma_4^t$, $\gamma_{12}^{\max} = \max_{t=1}^T \gamma_1^t + 2 \gamma_2^t $ and $\| \nabla g_t(\theta_t)\| \leq \mathcal{G}$ for all $t \in [T]$ in Theorem~\ref{thm: bound on invex dynamic regret}. Furthermore, define $V_T \coloneqq \sum_{t=1}^T \Delta(\theta_t^*, \theta_{t+1}^*)$. If the learner chooses $\alpha_t = \alpha$ such that
    \begin{align}
        \alpha = \sqrt{\frac{ R^2 + R \gamma_{12}^{\max} V_T}{ \mathcal{G}^2 \sum_{t=1}^T \gamma_3^t }}~,
    \end{align}
    then the following regret guarantee holds:
    \begin{align}
        \regret_{\id}^T(\btheta) &\leq  \gamma_4^{\max} \left( \sqrt{ (R^2 + R \gamma_{12}^{\max} V_T) \mathcal{G}^2 \sum_{t=1}^T \gamma_3^t }  \right)~.
    \end{align}
\end{corollary}

Observe that if $\sum_{t=1}^T \gamma_3^t
 = \calO(T)$, then $\regret_{\id}^T(\btheta) = \calO(\sqrt{V_T T})$.

\section{Bounds on the Clean Dynamic Regret in a Bounded Domain}
\label{sec: bounds on the clean regret}

\subsection{Proof of Theorem~\ref{thm: bounded domain clean dynamic regret}}
\label{subsec: thm: bound on clean static regret}
\begin{proof}
\label{proof thm: bound on clean static regret}

Consider $\| \theta \|_2 \leq D, \forall \theta \in \Theta$. We start our proof by analyzing the following quantity:
\begin{align}
    \label{eq: clean bounded domain dynamic regret}
    \| \theta_{t+1} - \theta^*_{t+1} \|^2 &= \| \theta_{t+1} - \theta_t^* + \theta_t^* - \theta^*_{t+1} \|^2 \\
    &= \| \theta_{t+1} - \theta_t^* \|^2 + \| \theta_t^* - \theta^*_{t+1} \|^2 + 2 \inner{\theta_{t+1} - \theta_t^*}{\theta_t^* - \theta^*_{t+1}}\\
    &\leq \underbrace{\| \theta_{t+1} - \theta_t^* \|^2}_{\text{Q1}} + \underbrace{ 6D \| \theta_t^* - \theta_{t+1}^* \| }_{\text{Q2}}
\end{align}

Observe that Q1 can be further decomposed as below.

\begin{align}
    \| \theta_{t+1} - \theta^*_t \|^2
    &= \| \Pi_{\Theta}(\theta_t - \alpha_t \nabla g_t(s_t, \theta_t)) - \theta^*_t \|^2 \\
    \label{eq: contraction}&\leq \| \theta_t - \alpha_t \nabla g_t(s_t, \theta_t) - \theta^*_t \|^2 \\
    &= \| \theta_t - \alpha_t \eta_t \nabla f_t(s_t, \theta_t) - \theta^*_t \|^2\\
    &\label{eq: base inequality}= \| \theta_t - \theta^*_t \|^2 + \alpha_t^2 \eta_t^2 \| \nabla f_t(s_t, \theta_t) \|^2 - 2 \alpha_t \eta_t \inner{\theta_t - \theta^*_t}{\nabla f_t(s_t, \theta_t)}
\end{align}

The first inequality~\eqref{eq: contraction} follows due to the contraction of the projection on the convex sets. We substitute inequality~\eqref{eq: base inequality} in inequality~\eqref{eq: clean bounded domain dynamic regret} while keeping track of corrupted and uncorrupted rounds.

First, using convexity of $f_t$ for uncorrupted samples, i.e, $t \in \calS$, we get
\begin{align}
\label{eq: uncorrupted_bound}
2\alpha_t \eta_t (f_t(s_t, \theta_t) - f_t(s_t, \theta^*_t)) &\leq   \| \theta_t - \theta^*_t \|^2 - \| \theta_{t+1} - \theta^*_{t+1} \|^2  + \alpha_t^2 \eta_t^2 \| \nabla f_t(s_t, \theta_t) \|^2 \notag\\
&\quad + 6D \| \theta_t^* - \theta_{t+1}^* \| \; ,
\end{align}

and using Cauchy–Schwarz inequality for corrupted samples, i.e., for $t \in [T] \backslash \calS $, we get
\begin{align}
\label{eq: corrupted_bound}
0 &\leq \| \theta_t - \theta^*_t \|^2 - \| \theta_{t+1} - \theta^*_{t+1} \|^2 + \alpha_t^2 \eta_t^2 \| \nabla f_t(s_t, \theta_t) \|^2 + 2 \alpha_t \eta_t \|\theta_t-\theta^*_t\| \| \nabla f_t(s_t, \theta_t) \| \notag \\
&\quad + 6D \| \theta_t^* - \theta_{t+1}^* \|~.
\end{align}

Using the results from Lemma~\ref{lem: bound on eta grad^r}, we can bound:
\begin{align}
    \eta_t \| \nabla f_t(s_t, \theta_t) \| \leq  G + \max( \frac{mL}{2ab}, \frac{4a^2L}{m^2 b} ) \coloneqq \conC~.
\end{align}

To bound the last term of inequality~\eqref{eq: corrupted_bound} consider the term $\eta_t \|\theta_t-\theta^*_t\| \| \nabla f_t(s_t, \theta_t)\|$ for some $t\in [T]\setminus \calS$.
\begin{align}
    \eta_t \| \theta_t-\theta^*_t\| \| \nabla f_t(s_t, \theta_t) \| &\leq \eta_t G \|\theta_t-\theta^*_t\|+  \eta_t L\| \theta_t - \omega_t^* \|\|\theta_t-\theta^*_t\|\\
    &\leq \eta_t G (\|\theta_t-\omega_t^*\|+\|\omega_t^*-\theta^*_t\|)+ \eta_t L \| \theta_t - \omega_t^* \|(\|\theta_t-\omega_t^*\|+\|\omega_t^*-\theta^*_t\|)\\
    &\label{eq: gradcorruption}\leq \eta_t G \|\theta_t-\omega_t^*\|+ \eta_t G \|\omega_t^*-\theta^*_t\|+  \eta_t L \| \theta_t - \omega_t^* \|^2 +   \eta_t L\|\omega_t^*\\
    &\quad -\theta^*_t\|\|\theta_t-\omega_t^*\|~.
\end{align}

Next, we will bound the term $\eta_t \| \theta_t - \omega_t^* \|^r, r \in \{1, 2\}$.
\begin{align*}
    \eta_t \| \theta_t - \omega_t^* \|^r &= \frac{\exp(-\frac{1}{a} f_t(s_t, \theta_t) )}{ \exp(-\frac{1}{a} f_t(s_t, \theta_t) ) + b} \| \theta_t - \omega_t^* \|^r  \\
    &\leq  \frac{1}{\exp( - \frac{1}{a} f_t(s_t, \theta_t) ) + b} \exp(-\frac{1}{a} ( f_t(s_t, \omega_t^*) )) \exp(-\frac{1}{a} \frac{m}{2} \| \theta_t - \omega_t^* \|^2) \| \theta_t - \omega_t^* \|^r \\
    &\leq \frac{1}{b}  \exp(-\frac{1}{a} \frac{m}{2} \| \theta_t - \omega_t^* \|^2) \| \theta_t - \omega_t^* \|^r~.
\end{align*}

Using the results from lemma~\ref{lem: exp trumps poly} with $x = \| \theta_t - \omega_t^* \|, c = \frac{m}{2a}$ and $s=2$,
\begin{align}
    \eta_t \| \theta_t - \omega_t^* \| &\leq \frac{1}{b}\max( \frac{m}{2a}, \frac{4a^2}{m^2} ) \coloneqq \conD~,\\
    \eta_t \| \theta_t - \omega_t^* \|^2 &\leq \frac{1}{b}\max( \frac{m}{2a}, \frac{2a}{m} ) \coloneqq \conE~.
\end{align}

Thus,
\begin{align}
    \eta_t \| \theta_t-\theta^*_t\| \| \nabla f_t(s_t, \theta_t) \| &\leq  G \conD+  G \|\omega_t^*-\theta^*_t\|+  L \conE +    L\|\omega_t^*-\theta^*_t\| \conD~.
\end{align}

Now we are ready to establish the sublinear bound on the clean dynamic regret.

We consider $\alpha_t = \alpha$. Observe that, $ \eta_t \geq \frac{1}{1 + b \exp(\frac{B}{a})}$ and $f_t(s_t, \theta_t) - f_t(s_t, \theta_t^*) \geq 0$ for uncorrupted rounds. Summing up and telescoping inequalities \eqref{eq: uncorrupted_bound} and \eqref{eq: corrupted_bound} from $t = 1, \ldots, T$, we get
\begin{align}
\begin{split}
\label{eq: before_regret}
    \regret_{\rd}^T(\btheta, \calS)
    &\leq  \left(1+b\exp\left(\frac{B}{a}\right)\right)\left(  \frac{\| \theta_1 - \theta^*_1 \|^2}{2\alpha} + \frac{\alpha}{2} \conC^2 T \right. \\
    &+ \left. kG \conD+  kG \max_{t \in [T] \setminus \calS}\|\omega_t^*-\theta^*_t\|+  k L \conE +    kL\conD \max_{t \in [T] \setminus \calS} \|\omega_t^*-\theta^*_t\| + \frac{6 D V_T}{2\alpha} \right)~.
\end{split}
\end{align}

Furthermore, observe that $ \| \theta_1 - \theta_1^* \| \leq 2D$. Thus,
\begin{align}
\begin{split}
\label{eq: before_regret1}
    \regret_{\rd}^T(\btheta, \calS)
    &\leq  \left(1+b\exp\left(\frac{B}{a}\right)\right)\left(  \frac{4D^2}{2\alpha} + \frac{\alpha}{2} \conC^2 T + kG \conD \right. \\
    &+ \left.  kG \max_{t \in [T] \setminus \calS}\|\omega_t^*-\theta^*_t\|+  k L \conE +    kL\conD \max_{t \in [T] \setminus \calS} \|\omega_t^*-\theta^*_t\| + \frac{6 D V_T}{2\alpha} \right)~.
\end{split}
\end{align}

By choosing,
\begin{align}
    \alpha = \sqrt{ \frac{ 4D^2 + 6 D V_T}{\conC^2 T } }~,
\end{align}
we get
\begin{align}
\begin{split}
\label{eq: final clean static regret}
    \regret_{\rd}^T(\btheta, \calS)
    &\leq  \left(1+b\exp\left(\frac{B}{a}\right)\right)\Big(  \conC  \sqrt{ (4D^2 + 6DV_T) T} +   k(G \conD+ L \conE) \notag \\
    &\quad +  k (G +  L\conD) \max_{t \in [T] \setminus \calS}\|\omega_t^*-\theta^*_t\|   \Big)~.
\end{split}
\end{align}

\end{proof}



\section{Bounds on the Clean Dynamic Regret in Unbounded Domain}
\label{sec: bounds on the clean dynamic regret}

\subsection{Formal Statement for Theorem~\ref{thm: clean dynamic regret bound}}
\label{subsec: formal statement thm: clean dynamic regret bound }

\begin{theorem}
    For each $t\in[T]$, let $\{ f_t(s_t, \cdot) \}_{t=1}^T$ be a sequence of $m$-strongly convex loss functions with possible $k$ of them corrupted by outliers. We assume that the loss for uncorrupted rounds falls within the range $[0, B]$, whereas the loss for corrupted rounds can be unbounded. Let $\btheta$ be a sequence of actions chosen by the learner using Algorithm~\ref{alg:outlier robust OGD expert} with parameters for the experts chosen from $\calE$ as defined in equation~\eqref{eq: expert parameters}. Let $\alpha^* = \sqrt{ \frac{32 \frakD^2 + 24 \frakD V_T}{T\conC^2}  }$. The following regret guarantees hold:
    \begin{itemize}
        \item If $\frakD \in [ \frac{2\epsilon}{T}, \frac{\epsilon2^T}{T}, ]$ and $\alpha^* \in [ \frac{2}{\sqrt{T}}, \frac{A_{\max}}{\sqrt{T}} ]$, then
        \begin{align}
    \regret_{\rd}^T(\btheta, \calS)
    & \leq \xi \Big(  3 \sqrt{ \frac{T \conC^2}{2} (16 \frakD^2 + 12 \frakD V_T)  } + k \big(G \conD +  L \conE   + \conC \max_{t \in [T] \setminus \calS} \|\omega_t^*-\theta^*_t\|\big)  + k \conF  \notag\\
    &\quad+ \conF \sqrt{\frac{T\log (T\log_2 A_{\max} )}{2}}  \Big)~.
        \end{align}
        \item If $\frakD \in [ \frac{2\epsilon}{T}, \frac{\epsilon2^T}{T}, ]$ and $\alpha^* \leq \frac{2}{\sqrt{T}}$, then
        \begin{align*}
     \regret_{\rd}^T(\btheta, \calS)
    & \leq  \xi \Big(  8 \frakD^2 \sqrt{T} +  \conC^2 \sqrt{T} + k \big(G \conD +  L \conE  \notag \\
    &\quad  + \conC \max_{t \in [T] \setminus \calS} \|\omega_t^*-\theta^*_t\|\big) + \conC \sqrt{(32 \frakD^2 + 24 \frakD V_T)T} + k \conF + \conF \sqrt{\frac{T\log (T\log_2 A_{\max} )}{2}}  \Big) \; .
        \end{align*}
        \item If $\frakD \in [ \frac{2\epsilon}{T}, \frac{\epsilon2^T}{T}, ]$ and $\alpha^* \geq \frac{A_{\max}}{\sqrt{T}}$, then
        \begin{align*}
    \regret_{\rd}^T(\btheta, \calS)
    & \leq  \xi \Big(  \frac{16 \frakD^2 }{C} +  \frac{\conC}{2} \sqrt{ (32 \frakD^2 + 24 \frakD V_T)T  } + k \big(G \conD +  L \conE  \notag \\
    &\quad  + \conC \max_{t \in [T] \setminus \calS} \|\omega_t^*-\theta^*_t\|\big) + \frac{12 \frakD V_T }{C}  + k \conF + \conF \sqrt{\frac{T\log (T\log_2 A_{\max} )}{2}}  \Big)~.
    \end{align*}
        \item If $\frakD <  \frac{2\epsilon}{T} $, then
        \begin{align}
    \regret_{\rd}^T(\btheta, \calS) &\leq \xi \Big( 8 \frac{\epsilon^2}{T\sqrt{T}} + \sqrt{T} \conC^2 + k \big(G \conD +  L \conE    + \conC \max_{t \in [T] \setminus \calS} \|\omega_t^*-\theta^*_t\|\big) + \frac{6 \epsilon V_T}{\sqrt{T}}  + k \conF \notag \\
    &\quad+ \conF \sqrt{\frac{T\log (T\log_2 A_{\max} )}{2}}  \Big)~.
        \end{align}
        \item If $\frakD >  \frac{\epsilon 2^T}{T} $, then
        \begin{align}
    \regret_{\rd}^T(\btheta, \calS) \leq 2 \frakD \big( G + 2 \frakD \big) \log_2 \frac{\frakD T}{\epsilon}~.
        \end{align}
    \end{itemize}
\end{theorem}

\subsection{Proof of Theorem~\ref{thm: clean dynamic regret bound}}
\label{subsec: proof of thm: clean dynamic regret bound}

\begin{proof}
 Let $\btheta^\tau = (\theta_1^\tau, \ldots, \theta_T^\tau)$ be a sequence of actions chosen by the expert $\tau$. Observe that for any expert $\tau$,
\begin{align}
    \regret_{\rd}^T(\btheta, \calS) &=  \sum_{t \in \calS} f_t(s_t, \theta_t) - \sum_{t \in \calS} f_t(s_t, \theta_t^*) \\
    &= \underbrace{\sum_{t \in \calS} f_t(s_t, \theta_t^\tau) - \sum_{t \in \calS} f_t(s_t, \theta_t^*)}_{\text{Meta Regret $\regret_{\rm{M}}^T(\btheta^{\tau}, \calS)$}} + \underbrace{\sum_{t \in \calS} f_t(s_t, \theta_t) - \sum_{t \in \calS} f_t(s_t, \theta_t^\tau)}_{\text{Expert Regret $\regret_{\rm{E}}^T(\btheta, \btheta^{\tau}, \calS)$}}
\end{align}

The proof unfolds in four major steps:
\begin{enumerate}
    \item Define $ \frakD \coloneqq \max_{t=1}^T \| \theta_t^* \|$. If $\frakD \leq D^{\tau}$ for any specific expert $\tau$, we establish that $\regret_{\rm{M}}^T(\btheta^{\tau}, \calS)$ is bounded.
    \item Demonstrate that  $\regret_{\rm{E}}^T(\btheta, \btheta^{\tau}, \calS)$ remains bounded when $\frakD \leq D^{\tau}$.
    \item Establish the bound on $\regret_{\rd}^T(\btheta, \calS)$ when $\frakD \geq D_{\max} \coloneqq \max_{\tau=1}^N D^{\tau}$.
    \item Conclude the proof by deriving a bound on $\regret_{\rd}^T(\btheta, \calS)$ when $\frakD \leq D_{\min} \coloneqq \min_{\tau=1}^N D^{\tau}$.
\end{enumerate}

First, we bound the meta regret in a bounded domain using the following lemma.
\begin{lemma}
    \label{lem: dynamic regret bounded domain}
    Let $\frakD \leq D^{\tau}$ for some expert $\tau$. Then,
    \begin{align}
        \label{eq:bounded meta regret}
        \regret_{\rm{M}}^T(\btheta^{\tau}, \calS) &\leq \xi \Big( \frac{ (D^\tau + \| \theta^*_1 \|)^2}{\alpha^\tau} + \frac{\alpha^\tau}{2} \sum_{t=1}^T \big( {\eta_t^\tau} \| \nabla f_t(s_t, \theta_t^\tau) \| \big)^2 + k \big(G \conD +  L \conE   \notag\\
        &\quad + \conC \max_{t \in [T] \setminus \calS} \|\omega_t^*-\theta^*_t\|\big) + \frac{6D^{\tau} V_T}{\alpha^\tau} \Big)
    \end{align}
\end{lemma}

We use the result from Lemma~\ref{lem: bound on expert regret} to bound the expert regret. In the next lemma, we will bound the clean dynamic regret when $\frakD$ is too large.

\begin{lemma}
    \label{lem: regret bound large comparator}
    Let $\frakD \geq \max_{\tau = 1}^N D^{\tau} \coloneqq D_{\max}$ and $f_t$ are a sequence of functions satisfying inequality~\eqref{eq: grad inequality}, then
    \begin{align}
        \label{eq: regret bound large comparator}
        \regret_{\rd}^T(\btheta, \calS) \leq 2 \frakD \big( G + 2 \frakD \big) \log_2 \frac{\frakD T}{\epsilon} \; .
    \end{align}
\end{lemma}

Finally, we will show that clean dynamic regret is bounded even when $\frakD$ is too small.

\begin{lemma}
    \label{lem: regret bound small comparator}
    Let $\frakD \leq D_{\min} = \min_{\tau = 1}^N D^{\tau}$, then the following regret guarantee holds:
    \begin{align}
    \label{eq: regret bound small comparator}
    \regret_{\rd}^T(\btheta, \calS) &\leq \xi \Big( 8 \frac{\epsilon^2}{T\sqrt{T}} + \sqrt{T} \conC^2 + k \big(G \conD +  L \conE    + \conC \max_{t \in [T] \setminus \calS} \|\omega_t^*-\theta^*_t\|\big) + \frac{6 \epsilon V_T}{\sqrt{T}}  + k \conF \notag \\
    &\quad+ \conF \sqrt{\frac{T\log (T\log_2 A_{\max} )}{2}}  \Big)
\end{align}
\end{lemma}

The final step is to bound the dynamic regret when $D_{\min} \leq \frakD \leq D_{\max}$, then it must be that for some expert $\tau$,
\begin{align*}
    \frac{D^\tau}{2} \leq \frakD \leq D^\tau \implies D^\tau \leq 2 \frakD \; .
\end{align*}

Using the analysis of Lemma~\ref{lem: dynamic regret bounded domain} and Lemma~\ref{lem: bound on expert regret}, we can write:
\begin{align*}
    \regret_{\rd}^T(\btheta, \calS)  &\leq \xi \Big( \frac{ (D^\tau + \| \theta^*_1 \|)^2}{\alpha^\tau} + \frac{\alpha^\tau}{2} \sum_{t=1}^T \big( {\eta_t^\tau} \| \nabla f_t(s_t, \theta_t^\tau) \| \big)^2 + k \big(G \conD +  L \conE  \notag \\
    &\quad  + \conC \max_{t \in [T] \setminus \calS} \|\omega_t^*-\theta^*_t\|\big) + \frac{6D^\tau V_T}{\alpha^\tau}  + k \conF + \conF \sqrt{\frac{T\log (T\log_2 A_{\max} )}{2}}  \Big)  \\
    & \leq \xi \Big( \frac{ (16 \frakD^2}{\alpha^\tau} + \frac{\alpha^\tau}{2} T \conC^2 + k \big(G \conD +  L \conE  \notag \\
    &\quad  + \conC \max_{t \in [T] \setminus \calS} \|\omega_t^*-\theta^*_t\|\big) + \frac{12 \frakD V_T}{\alpha^\tau}  + k \conF + \conF \sqrt{\frac{T\log (T\log_2 A_{\max} )}{2}}  \Big)
\end{align*}

The optimal $\alpha^\tau$ would be
\begin{align}
    \label{eq: ideal alpha}
    \alpha^* = \sqrt{ \frac{32 \frakD^2 + 24 \frakD V_T}{T\conC^2}  }
\end{align}

If $\alpha^* \leq \alpha_{\min}$, we choose $\alpha^\tau = \alpha_{\min}$, which results in the dynamic regret:
\begin{align*}
     \regret_{\rd}^T(\btheta, \calS)
    & \leq  \xi \Big(  8 \frakD^2 \sqrt{T} +  \conC^2 \sqrt{T} + k \big(G \conD +  L \conE  \notag \\
    &\quad  + \conC \max_{t \in [T] \setminus \calS} \|\omega_t^*-\theta^*_t\|\big) + \conC \sqrt{(32 \frakD^2 + 24 \frakD V_T)T} + k \conF + \conF \sqrt{\frac{T\log (T\log_2 A_{\max} )}{2}}  \Big) \; .
\end{align*}

Similarly, if $\alpha^* \geq \alpha_{\max}$, we choose $\alpha^\tau = \alpha_{\max}$, which results in the dynamic regret:
\begin{align*}
    \regret_{\rd}^T(\btheta, \calS)
    & \leq  \xi \Big(  \frac{16 \frakD^2 }{C} +  \frac{\conC}{2} \sqrt{ (32 \frakD^2 + 24 \frakD V_T)T  } + k \big(G \conD +  L \conE  \notag \\
    &\quad  + \conC \max_{t \in [T] \setminus \calS} \|\omega_t^*-\theta^*_t\|\big) + \frac{12 \frakD V_T }{C}  + k \conF + \conF \sqrt{\frac{T\log (T\log_2 A_{\max} )}{2}}  \Big)~.
\end{align*}

If $\alpha^*$ falls between $[\alpha_{\min}, \alpha_{\max}]$, then there exists an $\alpha^\tau$ such that $\alpha^\tau \leq \alpha^* \leq \alpha^{\tau + 1} = 2 \alpha^\tau$. We pick this $\alpha^\tau$ and bound the dynamic regret as:
\begin{align*}
    \regret_{\rd}^T(\btheta, \calS)
    & \leq \xi \Big(  3 \sqrt{ \frac{T \conC^2}{2} (16 \frakD^2 + 12 \frakD V_T)  } + k \big(G \conD +  L \conE   + \conC \max_{t \in [T] \setminus \calS} \|\omega_t^*-\theta^*_t\|\big)  + k \conF  \\
    &\quad + \conF \sqrt{\frac{T\log (T\log_2 A_{\max} )}{2}}  \Big)~.
\end{align*}

\end{proof}

\subsection{Proof of Lemma~\ref{lem: dynamic regret bounded domain}}
\label{subsec: dynamic regret in bounded domain}

\begin{proof}
    We proceed with the proof by analyzing the following quantity.
    \begin{align}
    \| \theta_{t+1}^\tau - \theta^*_{t+1} \|^2 &= \| \theta_{t+1}^\tau - \theta_t^* + \theta_t^* - \theta^*_{t+1} \|^2 \\
    &= \| \theta_{t+1}^\tau - \theta_t^* \|^2 + \| \theta_t^* - \theta^*_{t+1} \|^2 + 2 \inner{\theta_{t+1}^\tau - \theta_t^*}{\theta_t^* - \theta^*_{t+1}}\\
    &\leq \| \theta_{t+1}^\tau - \theta_t^* \|^2 + 6D^{\tau} \| \theta_t^* - \theta^*_{t+1} \|\\
    &= \| \Pi_{\Theta^\tau}(\theta_t^\tau - \alpha_t^\tau \nabla g_t(s_t, \theta_t^\tau)) - \theta^*_t \|^2 + 6D^{\tau} \| \theta_t^* - \theta^*_{t+1} \|\\
    &\leq \| \theta_t^\tau - \alpha_t^\tau \nabla g_t(s_t, \theta_t^\tau) - \theta^*_t \|^2 + 6D^{\tau} \| \theta_t^* - \theta^*_{t+1} \|\\
    &= \| \theta_t^\tau - \alpha_t^\tau \eta_t^\tau \nabla f_t(s_t, \theta_t^\tau) - \theta^*_t \|^2 + 6D^{\tau} \| \theta_t^* - \theta^*_{t+1} \|\\
    &\label{eq: dynamic base inequality}= \| \theta_t^\tau - \theta^*_t \|^2 + {\alpha_t^\tau}^2 {\eta_t^\tau}^2 \| \nabla f_t(s_t, \theta_t^\tau) \|^2 - 2 \alpha_t^\tau \eta_t^\tau \inner{\theta_t^\tau - \theta_t^*}{\nabla f_t(s_t, \theta_t^\tau)} \notag \\
    &\quad + 6D \| \theta_t^* - \theta^*_{t+1} \|
\end{align}

Using convexity of $f_t$ for uncorrupted samples, i.e, for $t \in \calS$,
\begin{align}
\label{eq: dynamic uncorrupted_bound}
2\alpha_t^\tau \eta_t^\tau (f_t(x_t, \theta_t^\tau) - f_t(x_t, \theta^*_t)) &\leq   \| \theta_t^\tau - \theta^*_t \|^2 - \| \theta_{t+1}^\tau - \theta^*_t \|^2  + {\alpha_t^\tau}^2 \big( {\eta_t^\tau} \| \nabla f_t(x_t, \theta_t^\tau) \| \big)^2 \notag\\
&\quad + 6D^{\tau} \| \theta_t^* - \theta^*_{t+1} \|
\end{align}

Similarly, for corrupted samples, i.e., $t \in [T] \backslash \calS $,
\begin{align}
\label{eq: dynamic corrupted_bound}
0 &\leq \| \theta_t^\tau - \theta^*_t \|^2 - \| \theta_{t+1}^\tau - \theta^*_t \|^2 + {\alpha_t^\tau}^2 \big( {\eta_t^\tau} \| \nabla f_t(y_t, \theta_t^\tau) \| \big)^2 + 2 \alpha_t^\tau \eta_t^\tau \|\theta_t^\tau-\theta^*_t\| \| \nabla f_t(y_t, \theta_t^\tau) \| \notag\\
&\quad + 6D^{\tau} \| \theta_t^* - \theta^*_{t+1} \|
\end{align}

Now, we follow the same argument as the dynamic regret bound in Theorem~\ref{thm: bounded domain clean dynamic regret} and end up with the following bound:
\begin{align}
     \regret_{\rm{M}}^T(\btheta^{\tau}, \calS) &\leq \Big(1+b\exp \Big( \frac{B}{a} \Big) \Big)\Big( \frac{ \| \theta_1^\tau - \theta^*_1 \|^2}{\alpha^\tau}  + \frac{\alpha^\tau}{2} \sum_{t=1}^T \big( {\eta_t^\tau} \| \nabla f_t(s_t, \theta_t^\tau) \| \big)^2 + k \big(G \conD +  L \conE  \notag \\
    &\quad  + \conC \max_{t \in [T] \setminus \calS} \|\omega_t^*-\theta^*_t\|\big) + \frac{3D^{\tau} V_T}{\alpha^\tau} \Big)
\end{align}
where $V_T = \sum_{t=1}^T \| \theta_t^* - \theta_{t+1}^* \|$. The above expression can be further simplified to:
\begin{align}
    \regret_{\rm{M}}^T(\btheta^{\tau}, \calS) &\leq \xi \Big( \frac{ (D^\tau + \| \theta^*_1 \|)^2}{\alpha^\tau} + \frac{\alpha^\tau}{2} \sum_{t=1}^T \big( {\eta_t^\tau} \| \nabla f_t(s_t, \theta_t^\tau) \| \big)^2 + k \big(G \conD +  L \conE  \notag\\
    &\quad + \conC \max_{t \in [T] \setminus \calS} \|\omega_t^*-\theta^*_t\|\big) + \frac{3D^{\tau} V_T}{\alpha^\tau} \Big)
\end{align}
\end{proof}

\subsection{Proof of Lemma~\ref{lem: bound on expert regret}}
\label{subsec: proof of lemma lem: bound on expert regret}
\begin{proof}
\label{proof: lem: bound on expert regret}
Recall from Algorithm~\ref{alg:outlier robust OGD expert} that,
\begin{align}
        \rho_{t+1}^\tau &= \rho_t^\tau \exp\big(-\beta \tilde{\eta}_t f_t(s_t, \theta_t^\tau)\big) \\
        &= \rho_1^\tau \exp\big( - \beta \sum_{t=1}^T \tilde{\eta}_t f_t(s_t, \theta_t^\tau) \big) \\
        &= \exp\big( - \beta \sum_{t \in \calS} \tilde{\eta}_t f_t(s_t, \theta_t^\tau) \big)  \exp\big( - \beta \sum_{t \in [T] \backslash \calS} \tilde{\eta}_t f_t(s_t, \theta_t^\tau) \big)
\end{align}
Furthermore, using the definition of $Z_t$:
\begin{align}
        \log \frac{Z_{T+1}}{Z_{1}} &= \log \frac{\sum_{\tau =1}^N \rho_{t+1}^\tau}{N} \\
        &= \log \Big(\sum_{\tau =1}^N \exp\big( - \beta \sum_{t \in \calS} \tilde{\eta}_t f_t(s_t, \theta_t^\tau) \big)  \exp\big( - \beta \sum_{t \in [T] \setminus \calS} \tilde{\eta}_t f_t(s_t, \theta_t^\tau) \big) \Big) - \log N\\
        \label{eq: eta f_t bound}&\geq \log\Big( \exp\big( - \beta \sum_{t \in \calS} \tilde{\eta}_t f_t(s_t, \theta_t^\tau) \big)\Big) + \log \Big(\exp(-\beta k \conF)\Big) - \log N \\
        \label{eq: lower bound log Z_T+1/Z_1}&\geq  - \beta  \sum_{t \in \calS} \tilde{\eta}_t  f_t(s_t, \theta_t^\tau) -\beta k \conF - \log N \; ,
\end{align}

where $\conF \coloneqq \frac{1}{b} \max(a, \frac{1}{a}) $. The inequality~\eqref{eq: eta f_t bound} follows from Lemma~\ref{lem: bound on eta f}.

Now, consider
\begin{align}
    \log \frac{Z_{t+1}}{Z_t} &= \log \frac{\sum_{\tau=1}^N \rho^\tau_{t+1}}{\sum_{\tau=1}^N \rho^\tau_t} \\
    &= \log \frac{\sum_{\tau=1}^N \exp\big(-\beta \tilde{\eta}_t f_t(s_t, \theta_t^\tau)\big) \rho^\tau_t}{\sum_{\tau=1}^N \rho^\tau_t}
\end{align}

Note that the following inequality holds for any random variable $X \in [a, b]$ and $s > 0$:
\begin{align}
    \log \E[\exp(s X)] \leq s\E[X] + \frac{s^2 (b - a)^2}{8}
\end{align}

Consider $X = - \tilde{\eta}_t f_t(s_t, \theta_t^\tau), s = \beta$. Also, observe that $X \in [-\conF, 0]$. Thus,
\begin{align}
    \log \frac{Z_{t+1}}{Z_t} &\leq - \beta \frac{ \sum_{\tau=1}^N \tilde{\eta}_t f_t(s_t, \theta_t^\tau) \rho_{t}^\tau}{Z_t} + \beta^2 \frac{\conF^2}{8}
\end{align}
Using the convexity of $f_t(s_t, \theta_t^\tau)$,
\begin{align}
    \log \frac{Z_{t+1}}{Z_t} &\leq - \beta \tilde{\eta}_t  f_t(s_t, \theta_t) + \beta^2 \frac{\conF^2}{8}
\end{align}

If round $t$ is corrupted, then
\begin{align}
    \log \frac{Z_{t+1}}{Z_t} &\leq  \beta^2 \frac{\conF^2}{8}
\end{align}

Summing through $t=1$ to $T$, we get
\begin{align}
    \label{eq: upper bound log Z_T+1/Z_1}
    \log \frac{Z_{T+1}}{Z_1} &\leq - \beta \sum_{t \in \calS} \tilde{\eta}_t  f_t(s_t, \theta_t) + \beta^2 T \frac{\conF^2}{8}
\end{align}

Using equations~\eqref{eq: lower bound log Z_T+1/Z_1} and \eqref{eq: upper bound log Z_T+1/Z_1}:
\begin{align}
    - \beta  \sum_{t \in \calS} \tilde{\eta}_t  f_t(s_t, \theta_t^\tau) -\beta k \conF - \log N \leq - \beta \sum_{t \in \calS} \tilde{\eta}_t  f_t(s_t, \theta_t) + \beta^2 T \frac{\conF^2}{8}
\end{align}

It follows that,
\begin{align}
   \regret_{\rm{E}}^T(\btheta, \btheta^{\tau}, \calS) \leq  \Big(1 + b \exp\Big(\frac{B}{a}\Big)\Big) \Big( k \conF + \frac{\log N}{\beta}   + \beta T \frac{\conF^2}{8} \Big)
\end{align}

By picking $\beta = \sqrt{\frac{ 8 \log N }{T \conF^2}}$, we get
\begin{align}
  \label{eq: expert regret}
  \regret_{\rm{E}}^T(\btheta, \btheta^{\tau}, \calS) \leq  \Big(1 + b \exp\Big(\frac{B}{a}\Big)\Big)  \Big( k \conF + \conF \sqrt{\frac{T\log N}{2}}    \Big)
\end{align}

\end{proof}

\subsection{Proof of Lemma~\ref{lem: regret bound large comparator}}
\label{subsec: proof of lemma lem: regret bound large comparator}
\begin{proof}
\label{proof: lem: regret bound large comparator}
In this setting, we provide a trivial bound on dynamic regret (similar to \citet{jacobsen2023unconstrained}).
\begin{align}
    \sum_{t \in \calS} f_t(s_t, \theta_t) - \sum_{t \in \calS} f_t(s_t, \theta_t^*) &\leq \sum_{t \in \calS} \| \nabla f_t(\cdot, \theta_t) \| \| \theta_t - \theta_t^* \| \\
    &\leq \sum_{t \in \calS} \| \nabla f_t(s_t, \theta_t) \|  (D_{\max} + \frakD) \\
    &\leq 2 \frakD (G + 2 \frakD) T
\end{align}

Note that due to our construction of set $\calE_2$,  $D_{\max} = \frac{\epsilon 2^T}{T}$. Clearly, $T \leq \log_2 \frac{\frakD T}{\epsilon}$. Thus,
\begin{align}
\label{eq: dynamic regret with big theta*t}
     \regret_{\rd}^T(\btheta, \calS) &\leq  2 \frakD (G + 2 \frakD)  \log_2 \frac{\frakD T}{\epsilon}
\end{align}
\end{proof}

\subsection{Proof of Lemma~\ref{lem: regret bound small comparator}}
\label{subsec: proof of lemma lem: regret bound small comparator}

\begin{proof}
\label{proof: lem: regret bound small comparator}
Now we consider the setting when $\frakD \leq D_{\min}$. Note that since for all $t \in [T]$ we have $\| \theta_t^* \| \leq D_{\min}$, we can use the analysis of Lemma~\ref{lem: dynamic regret bounded domain} along with Lemma~\ref{lem: bound on expert regret} for the expert with parameters $(\alpha_{\min}, D_{\min})$. Formally,
\begin{align}
    \regret_{\rd}^T(\btheta, \calS)  &\leq \xi \Big( \frac{ (D_{\min} + \| \theta^*_1 \|)^2}{\alpha_{\min}} + \frac{\alpha_{\min}}{2} \sum_{t=1}^T \big( {\eta_t^\tau} \| \nabla f_t(s_t, \theta_t^\tau) \| \big)^2 + k \big(G \conD +  L \conE  \notag \\
    &\quad  + \conC \max_{t \in [T] \setminus \calS} \|\omega_t^*-\theta^*_t\|\big) + \frac{6D_{\min} V_T}{\alpha_{\min}}  + k \conF + \conF \sqrt{\frac{T\log N}{2}}  \Big)
\end{align}

Recall that,
\begin{align}
    \label{eq: bound on N}
    N \leq T\log_2 A_{\max}
\end{align}
and
\begin{align}
    \label{eq: bound on D_min/alpha_min}
    \frac{D_{\min}}{\alpha_{\min}} = \frac{\epsilon}{\sqrt{T}}
\end{align}

Using the above inequalities, we bound the dynamic regret as
\begin{align}
    \regret_{\rd}^T(\btheta, \calS) &\leq  \xi \Big( 8 \frac{\epsilon^2}{T\sqrt{T}} + \sqrt{T} \conC^2 + k \big(G \conD +  L \conE    + \conC \max_{t \in [T] \setminus \calS} \|\omega_t^*-\theta^*_t\|\big) + \frac{6 \epsilon V_T}{\sqrt{T}}  + k \conF \notag \\
    &\quad + \conF \sqrt{\frac{T\log (T\log_2 A_{\max} )}{2}}  \Big)
\end{align}
\end{proof}

\section{Useful bounds}
\label{sec: useful bounds}
Within this section, we derive some critical bounds integral to our analytical framework.

\begin{lemma}
\label{lem: exp trumps poly}
    For all $c,r,s> 0$ and $x \geq c^{\frac{1}{r}}$, $\exp(-cx^s)x^r \leq \frac{1}{x^s}\leq \frac{1}{c^\frac{s}{r}}$.
\end{lemma}
\begin{proof}
      We will show that for $x \geq c^{\frac{1}{r}}$, $x^r \exp(-c x^s) \leq \frac{1}{x^s}$ holds. Below, we assume $x^r \exp(-c x^s) \leq \frac{1}{x^s}$ and find a range of $x$ that satisfies this assumption.
    \begin{align*}
        x^r \exp(-c x^s) \leq \frac{1}{x^s}
        \implies r \log x - cx^s \leq -s \log x
        \implies cx^s \geq \log x^{r + s}
    \end{align*}
    Take $x^r \geq c$. As long as $x > 0$, the above inequality is valid.
\end{proof}

\begin{lemma}
    \label{lem: bound on eta grad^r}
    Let $f_t(s_t, \theta)$ be an $m$-strongly convex function whose gradient follows inequality~\eqref{eq: grad inequality}. Then,
    \begin{align}
        \label{eq: bound on eta grad^r}
        \eta_t \| \nabla f_t(s_t, \theta_t) \| \leq \conC \coloneqq  G + \max( \frac{mL}{2ab}, \frac{4a^2L}{m^2 b} ) \;,
    \end{align}
    where $\eta_t$ is defined according to Equation~\eqref{eq: eta definition}.
\end{lemma}
\begin{proof}
\label{proof: bound on eta grad^r}

We begin our proof by analyzing the following quantity.
\begin{align}
    \eta_t \| \nabla f_t(s_t, \theta_t) \| &=  \frac{ \exp( - \frac{1}{a} f_t(s_t, \theta_t) )  }{\exp( - \frac{1}{a} f_t(s_t, \theta_t) ) + b} \| \nabla f_t(s_t, \theta_t) \|
\end{align}

Note that due to $m$-strong convexity of $f_t(s_t, \cdot)$, we have
\begin{align}
    f_t(s_t, \omega_t^*) + \inner{\nabla f_t(s_t, \omega_t^*)}{ \theta_t - \omega_t^* } + \frac{m}{2} \| \theta_t - \omega_t^*  \|^2 \leq f_t(s_t, \theta_t)
\end{align}
where $\omega_t^* = \arg\min_{\theta \in \Theta} f_t(s_t, \theta)$. Observe that $\inner{\nabla f_t(s_t, \omega_t^*)}{\theta_t - \omega_t^*} \geq 0$. Thus,
\begin{align}
    f_t(s_t, \omega_t^*) + \frac{m}{2} \| \theta_t - \omega_t^*  \|^2 \leq f_t(s_t, \theta_t)
\end{align}

This leads to the following inequality:
\begin{align}
\begin{split}
    \eta_t \| \nabla f_t(s_t, \theta_t) \| &\leq \eta_t G + L  \frac{\exp(-\frac{1}{a} ( f_t(s_t, w_t^*) )) \exp(-\frac{1}{a} \frac{m}{2} \| \theta_t - \omega_t^* \|^2) \| \theta_t - \omega_t^* \|}{\exp( - \frac{1}{a} f_t(s_t, \theta_t) ) + b} \\
    &\leq G +  \frac{L}{\exp( - \frac{1}{a} f_t(s_t, \theta_t) ) + b}  \exp(-\frac{1}{a} \frac{m}{2} \| \theta_t - \omega_t^* \|^2) \| \theta_t - \omega_t^* \|\\
    &\leq G +  \frac{L}{b}  \exp(-\frac{1}{a} \frac{m}{2} \| \theta_t - \omega_t^* \|^2) \| \theta_t - \omega_t^* \|
\end{split}
\end{align}

As the first case, let $\| \theta_t - \omega_t^* \| < \frac{m}{2a}$, then
\begin{align}
\begin{split}
    \eta_t \| \nabla f_t(s_t, \theta_t) \| &\leq G + \frac{mL}{2ab} \; .
\end{split}
\end{align}
Now, let $\| \theta_t - \omega_t^* \| \geq \frac{m}{2a}$, then using the results from Lemma~\ref{lem: exp trumps poly}, we have
\begin{align}
    \begin{split}
        \eta_t \| \nabla f_t(s_t, \theta_t) \| &\leq G + \frac{4a^2L}{m^2b}
    \end{split}
\end{align}

It follows that,
\begin{align}
    \eta_t \| \nabla f_t(s_t, \theta_t) \| &\leq \conC \coloneqq G + \max( \frac{mL}{2ab}, \frac{4a^2L}{m^2 b} ) \; .
\end{align}
\end{proof}

\begin{lemma}
    \label{lem: bound on eta f}
    Let $f_t(s_t, \theta)$ be a function whose gradient follows inequality~\eqref{eq: grad inequality}. Then,
    \begin{align}
        \label{eq: bound on eta f}
        \eta_t f_t(s_t, \theta_t) \leq \frac{1}{b} \max(a, \frac{1}{a}) \coloneqq \conF  \;,
    \end{align}
    where $\eta_t$ is defined according to Equation~\eqref{eq: eta definition}.
\end{lemma}
\begin{proof}
    \label{proof: bound on eta f}
    We follow a similar approach as Lemma~\ref{lem: bound on eta grad^r}:
    \begin{align}
        \eta_t f_t(s_t, \theta_t) &= \frac{\exp( -\frac{1}{a} f_t(s_t, \theta_t) )}{\exp( -\frac{1}{a} f_t(s_t, \theta_t) ) + b} f_t(s_t, \theta_t) \\
        &\leq \frac{1}{b} \exp( -\frac{1}{a} f_t(s_t, \theta_t) ) f_t(s_t, \theta_t)
    \end{align}
    Note that if $f_t(s_t, \theta_t) \leq \frac{1}{a}$, then
    \begin{align}
        \eta_t f_t(s_t, \theta_t) \leq \frac{1}{ab} \;.
    \end{align}
    Using the results from Lemma~\ref{lem: exp trumps poly}, if $f_t(s_t, \theta_t) \geq \frac{1}{a}$, then
    \begin{align}
        \eta_t f_t(s_t, \theta_t) \leq \frac{a}{b} \;.
    \end{align}
    Thus,
    \begin{align}
        \eta_t f_t(s_t, \theta_t) \leq \frac{1}{b} \max(a, \frac{1}{a}) \coloneqq \conF \; .
    \end{align}
\end{proof}

\section{Experimental Validation}
\label{sec: experimental validation}
In this section, we perform numerical experiments to substantiate our theoretical findings. The experiments were conducted on a MacBook Pro running macOS 14.4.1, equipped with 32 GB of memory and an Apple M2 Max chip.

\subsection{Baseline Algorithms}
\label{subsec:baselines}

We conduct a comparative analysis, pitting \Learn{} against several baseline algorithms, each serving as a benchmark in our experiments.

\begin{itemize}
\item Top-k Filter Algorithm: Developed by~\citet{van2021robust}, this algorithm operates under the assumption that the domain of the Online Convex Optimization (OCO) problem and the gradients of the uncorrupted loss functions $f_t$ are bounded. The core concept involves filtering out $k$ rounds with the largest gradients, assuming they might be outliers. If any outlier round is not filtered, it does not influence the update too much due to the bounded gradient assumption on the uncorrupted loss. The algorithm requires precise knowledge of $k$ and provides an $\calO(\sqrt{T} + k)$ bound on clean static regret under the bounded domain and Lipschitz gradient assumption. While this algorithm is analyzed for static regret, it can be extended easily to the setting of dynamic regret.
\item Uncertain Top-k Filter Algorithm: We devise a variant of the Top-k filter algorithm equipped with an estimate of the number of outliers. This estimate was fixed at $0.75k$. It is important to emphasize that this manufactured algorithm does not come with any accompanying theoretical guarantees.
\item Online Gradient Descent (OGD) Algorithm: As our third baseline, we select OGD, a workhorse for many machine learning problems. OGD operates without explicit consideration of the presence of outliers.
\end{itemize}

In the absence of outliers, all baselines converge to Online Gradient Descent (OGD). It is imperative to note that none of these methods provide theoretical guarantees in an unbounded domain with outliers. Despite the absence of theoretical guarantees, we selected these algorithms as baselines due to a shortage of alternatives that effectively address both unbounded domains and robustness to outliers.


\subsection{Online Support Vector Machine}
\label{subsec:online-svm}

\begin{figure*}[!ht]
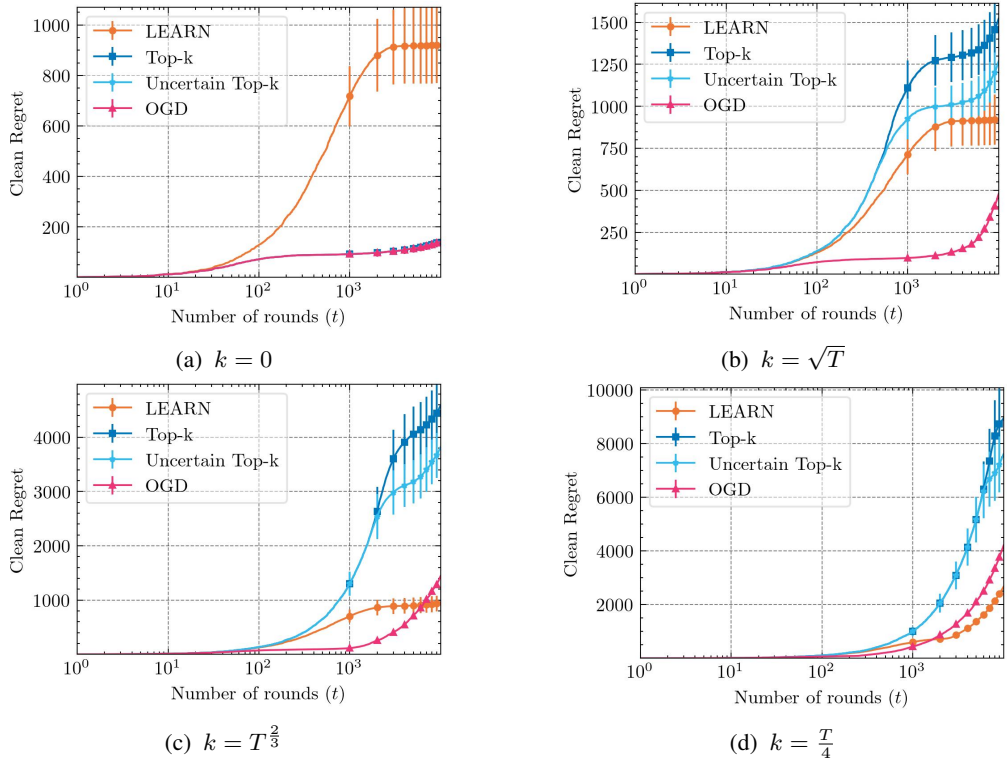

    \centering
    \begin{subfigure}{0.45\textwidth}
    \centering
    \includegraphics[scale=0.7]{svm_clean_kcfr_log.jpg}
        \caption{\label{fig: svm a} $k = 0$ }
    \end{subfigure}%
    \begin{subfigure}{0.45\textwidth}
    \centering
    \includegraphics[scale=0.7]{svm_sqrt_kcfr_log.jpg}
        \caption{\label{fig: svm b} $k = \sqrt{T}$ }
    \end{subfigure}
    \begin{subfigure}{0.45\textwidth}
    \centering
    \includegraphics[scale=0.7]{svm_twothird_kcfr_log.jpg}
        \caption{\label{fig: svm c} $k = T^{\frac{2}{3}}$ }
    \end{subfigure}%
    \begin{subfigure}{0.45\textwidth}
    \centering
    \includegraphics[scale=0.7]{svm_constant_kcfr_log.jpg}
        \caption{\label{fig: svm d} $k = \frac{T}{4}$ }
    \end{subfigure}
    \caption{\label{fig:  Clean regret svm} Clean Regret for Online SVM in Presence of Varying Number of Outliers ($k$).}
\end{figure*}

In our initial experimental configuration, we conducted a comparative analysis of \Learn{} against other baseline methods for a classification problem, employing soft-margin Support Vector Machines (SVM) in an online setting. The specifics of the experimental setup are detailed below:

\paragraph{Uncorrupted Data Generation}

In our experiments, the data for uncorrupted rounds ($t \in \calS$) was generated following a specific data generation process:
\begin{align}
\label{eq: data generation process}
y_t = \text{sign}\big( \inner{\theta^*}{x_t} \big)
\end{align}

Here, the entries of $\theta^* \in \real^2$ were uniformly drawn at random from the interval $[1, 11]$. Note that $s_t = (x_t, y_t)$ is the side information in this case. The entries of $x_t$ for uncorrupted rounds were independently drawn from a normal distribution $\mathcal{N}(0, 100)$. To allow for some mislabeling, the label of a round was flipped with a probability of $0.05$ if $\mid \inner{\theta^*}{x_t} \mid \leq 0.1$.

\paragraph{Corrupted Data Generation}

In the experimental setup, we executed a total of $10^4$ rounds, allowing the adversary the flexibility to corrupt any $k$ out of the $T$ rounds, with the choice made uniformly at random. The experiments encompassed scenarios where $k$ took on values from the set $\Big\{0, \sqrt{T}, T^{\frac{2}{3}}, \frac{T}{4}\Big\}$.

For the corrupted rounds, the labeling process involved flipping the label, specifically by setting $y_t = - \text{sign}\big( \inner{\theta^*}{x_t} \big)$.

\paragraph{Strongly Convex Loss Function}

Drawing inspiration from~\citet{shalev2007pegasos}, we adopted a fixed loss function $f_t$ for each round, defined as:
\begin{align}
\label{eq:online svm loss}
f_t\big( (x_t, y_t), \theta \big) = \frac{\lambda}{2} \| \theta \|^2 + \max(0, 1 - y_t \inner{x_t}{\theta})~.
\end{align}

Similar to \citet{shalev2007pegasos}, we set the parameter $\lambda$ to $10^{-4}$. To maintain an unbounded domain for $\theta$, we considered it as belonging to $\real^2$. It should be noted that in a noiseless setting with sufficiently small $\lambda$, dynamic optimal actions $\theta_t^*$ match the ground truth $\theta^*$.

\paragraph{Choice of Parameters}

In the case of \Learn{}, the parameter $a$ was set to $10^4$, while $b$ was maintained at $10$. This selection was determined through a grid search across a small range of values. We did not undertake specific tuning efforts for optimality. Across all methods, including both baselines and \Learn{}, a consistent step size of $\alpha = \frac{1}{\sqrt{T}}$ was employed.

\paragraph{Results}

The experiments were executed across 30 independent runs, and the outcomes are presented in Figure~\ref{fig:  Clean regret svm}. Referring back to Theorem~\ref{thm: clean dynamic regret bound}, we anticipate the clean regret to vary in the order of $\calO(\sqrt{T} + k)$. This implies sublinear regret for \Learn{} until $k = T^{\frac{2}{3}}$ and constant regret when $k = \frac{T}{4}$. The same behavior can be observed from the plots in Figure~\ref{fig:  Clean regret svm}. The empirical observations derived from the experiment are outlined below:
\begin{itemize}
    \item \Learn{} adopts a cautious updating strategy in the uncorrupted regime: This is illustrated by the observed gap in clean regret in Figure \ref{fig: svm a}. In this context, the clean regret for other baselines flattens at a significantly lower value compared to \Learn{}. This behavior stems from \Learn{}'s inherent cautious update mechanism, specifically designed to anticipate the presence of corrupted rounds. Consequently, \Learn{} experiences a substantial accumulation of regret initially, navigating a careful trajectory until it converges to the optimal action, leading to the eventual flattening of the regret curve.
    \item \Learn{} exhibits robustness to outliers: With an increasing value of $k$, the performance of the baselines markedly deteriorates, while \Learn{} remains relatively unaffected until Figure~\ref{fig: svm d}. This is evident when inspecting the maximum regret values across various curves. For \Learn{}, the maximum regret hovers around $900$ until Figure~\ref{fig: svm d}, whereas other baselines experience a steep surge in regret in the presence of outliers.
    \item \Learn{} is good at capturing the underlying ground truth: In cases where the data adheres to a specific generating process, the presence of a flat curve indicates convergence to the optimal action ($\theta^*$). Remarkably, \Learn{} maintains a flat region in Figures~\ref{fig: svm a} to \ref{fig: svm c} even in the presence of outliers, showcasing its ability to learn and adapt to the true underlying dynamics. In contrast, other methods falter in achieving this sustained flatness. Upon closer examination of the rate of growth, it becomes apparent that OGD is particularly susceptible to the disruptive influence of outliers.
    \item A constant proportion of outliers proves detrimental for all methods: As anticipated by Theorem~\ref{thm: clean dynamic regret bound}, \Learn{} begins to accumulate increasing regret when $k = \frac{T}{4}$. This phenomenon is evident in Figure~\ref{fig: svm d}. In fact, all the baseline methods exhibit poor performance in this regime.
\end{itemize}

\paragraph{Visualizing the Decision Boundary}

\begin{figure*}[!ht]
    \centering
    \begin{subfigure}{0.45\textwidth}
    \includegraphics[scale=0.85]{svm_clean_decision.jpg}
        \caption{\label{fig: svm db a} $k = 0$ }
    \end{subfigure}%
        \begin{subfigure}{0.45\textwidth}
    \includegraphics[scale=0.85]{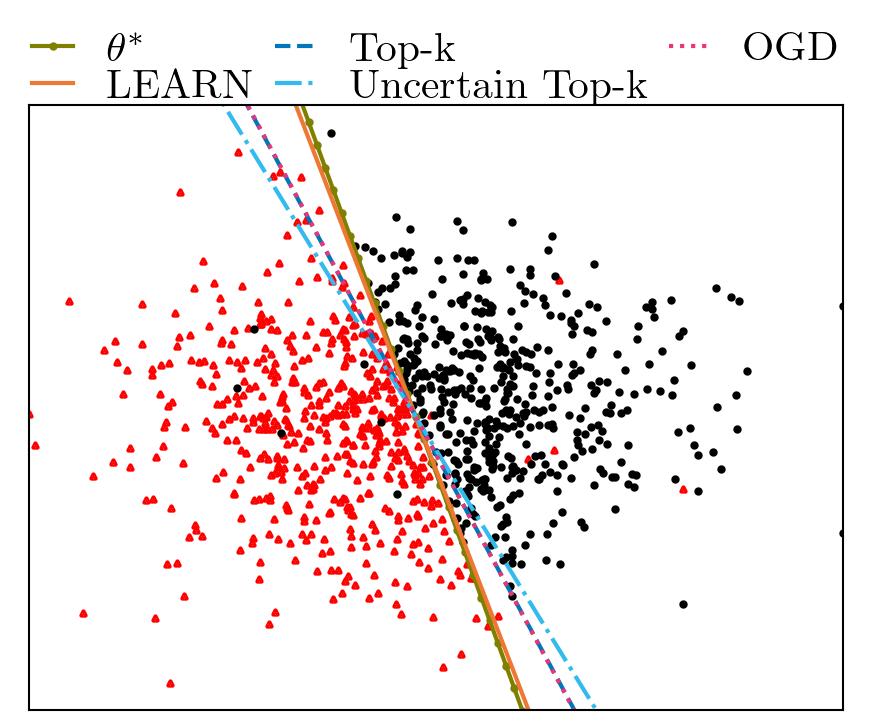}
        \caption{\label{fig: svm db b} $k = \sqrt{T}$ }
    \end{subfigure}
    \begin{subfigure}{0.45\textwidth}
    \includegraphics[scale=0.85]{svm_twothird_decision.jpg}
        \caption{\label{fig: svm db c} $k = T^{\frac{2}{3}}$ }
    \end{subfigure}%
    \begin{subfigure}{0.45\textwidth}
    \includegraphics[scale=0.85]{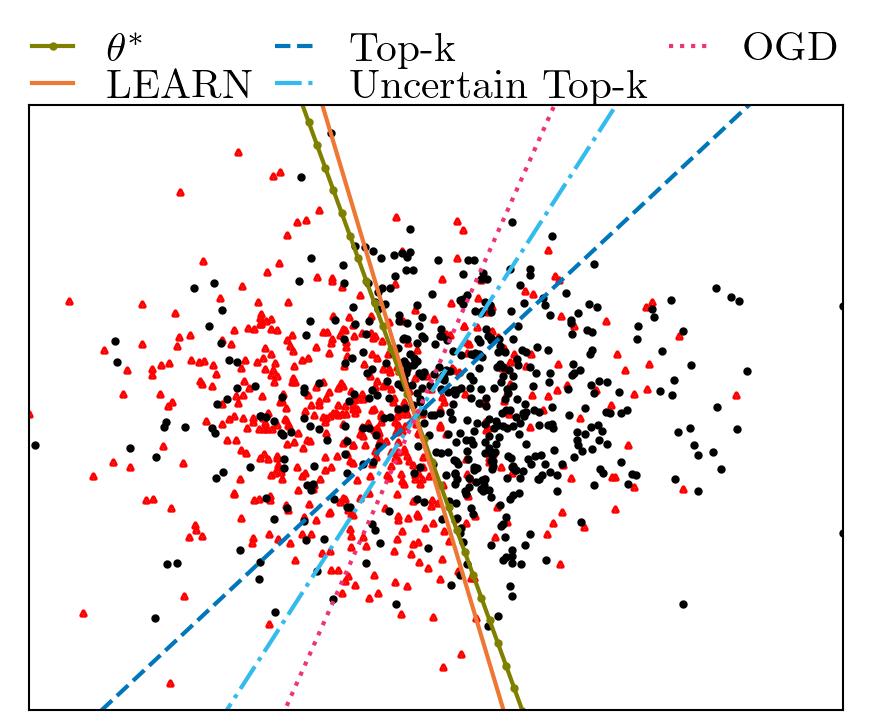}
        \caption{\label{fig: svm db d} $k = \frac{T}{4}$ }
    \end{subfigure}
    \caption{\label{fig:  decision boundary for svm} Visualization of Decision Boundary in Presence of Varying Number of Outliers ($k$).}
\end{figure*}

As the data (side information) adheres to a noiseless generative model, the online SVM problem can be thought of as a learning task where the objective is to learn the classification parameter $\theta^*$. The decision boundaries learned by different methods with varying numbers of outliers are illustrated in Figure \ref{fig: decision boundary for svm} for a single run. The decision boundary for $\theta^*$ (plotted as $\inner{\theta^*}{x} = 0$) is also depicted. To enhance clarity, only $500$ rounds out of the total $T=10^4$ rounds are randomly selected for plotting.

In scenarios with uncorrupted data, all methods accurately learn $\theta^*$. However, as the number of corrupted rounds increases, the decision boundaries for other baselines gradually deviate from the optimal boundary. In contrast, \Learn{} exhibits robustness against outliers, maintaining a decision boundary close to the optimal configuration even in the presence of corrupted data.

\subsection{Online Linear Regression}
\label{subsec:online-linear-regression}

\begin{figure*}[!ht]
    \centering
    \begin{subfigure}{0.45\textwidth}
    \centering
    \includegraphics[scale=0.7]{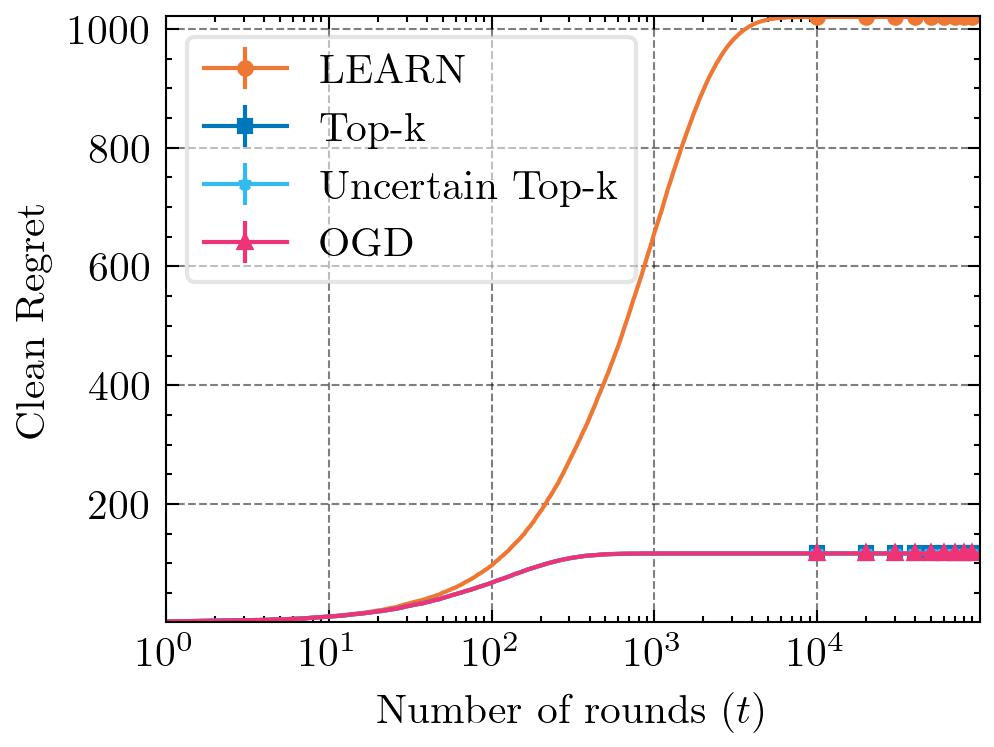}
        \caption{\label{fig: lin reg a} $k = 0$ }
    \end{subfigure}%
        \begin{subfigure}{0.45\textwidth}
    \centering
    \includegraphics[scale=0.7]{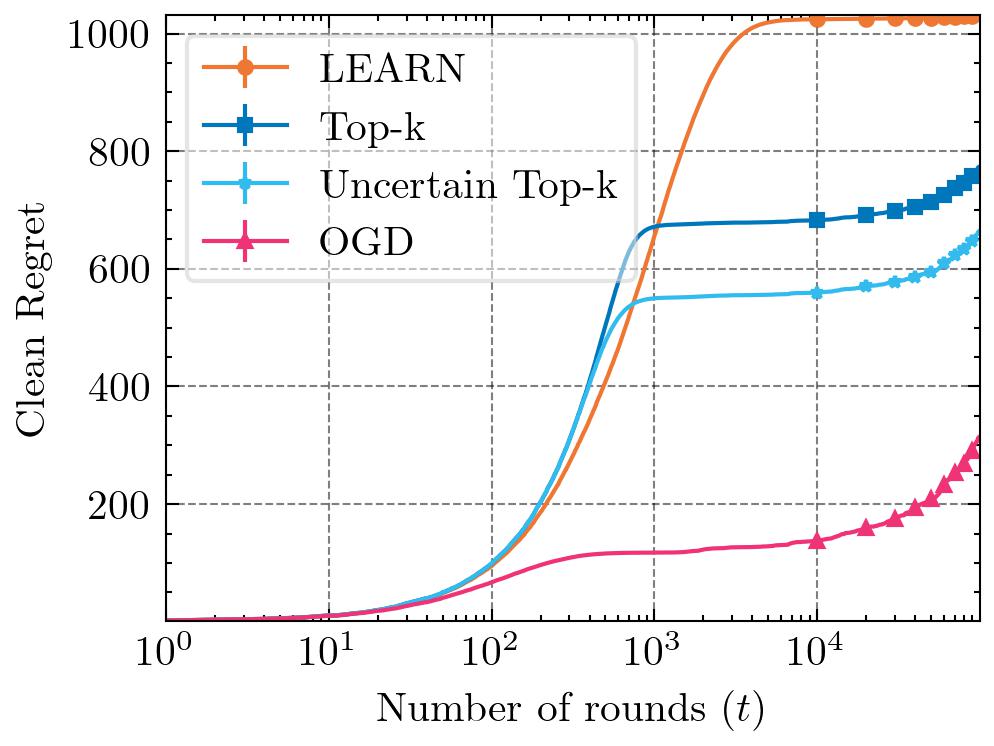}
        \caption{\label{fig: lin reg b} $k = \sqrt{T}$ }
    \end{subfigure}
    \begin{subfigure}{0.45\textwidth}
    \centering
    \includegraphics[scale=0.7]{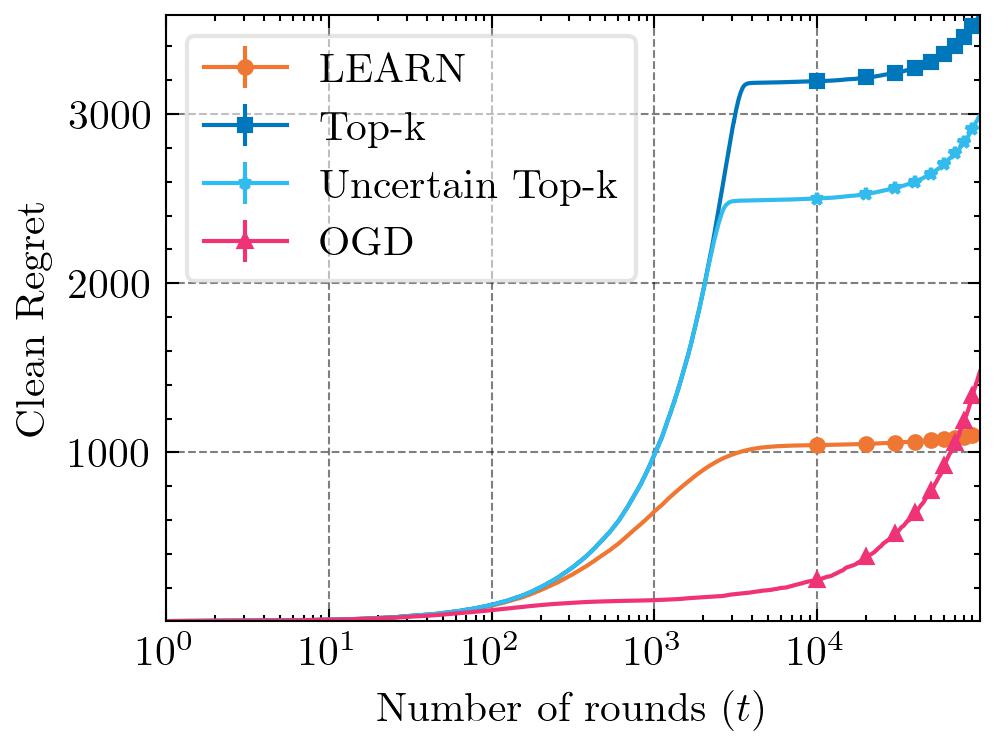}
        \caption{\label{fig: lin reg c} $k = T^{\frac{2}{3}}$ }
    \end{subfigure}%
    \begin{subfigure}{0.45\textwidth}
    \centering
    \includegraphics[scale=0.7]{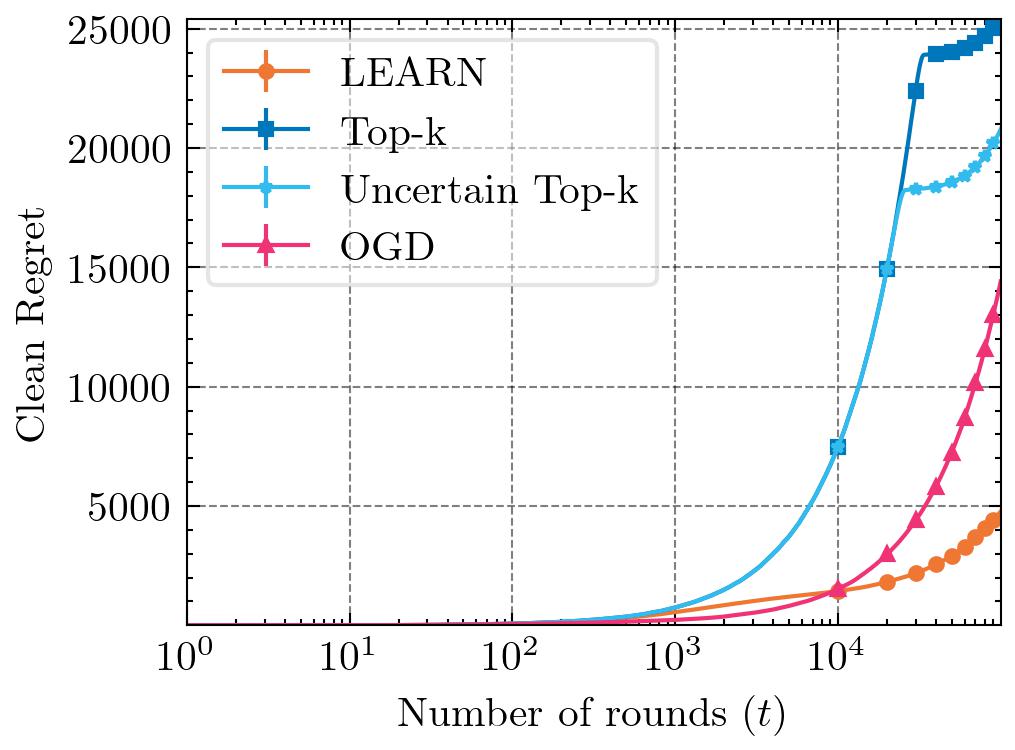}
        \caption{\label{fig: lin reg d} $k = \frac{T}{4}$ }
    \end{subfigure}
    \caption{\label{fig:  Clean regret for lin reg} Clean Dynamic Regret for Online Linear Regression in Presence of Varying Number of Outliers ($k$).}
\end{figure*}

In our subsequent experimental setup, we undertook a comparative analysis, juxtaposing the performance of \Learn{} against other baseline methods for a regression problem. Specifically, we explored the online ridge regression setting. The details of this experimental configuration are elaborated below:

\paragraph{Uncorrupted Data Generation}

The data for uncorrupted rounds ($t \in \calS$) adhered to a specific data generation process:
\begin{align}
\label{eq:lin-reg-data-generation-process}
y_t = \inner{\theta^*}{x_t} + e_t
\end{align}

In this context, the entries of $\theta^* \in \real^{100}$ were uniformly drawn at random from the interval $[-1, 1]$, followed by normalization to a unit norm. Note that $s_t = (x_t, y_t)$ is again the side information. The entries of $x_t$ for uncorrupted rounds were independently drawn from a normal distribution $\mathcal{N}(0, 1)$. Additionally, a small additive noise $e_t$ was independently chosen from $\mathcal{N}(0, 10^{-6})$.

\paragraph{Corrupted Data Generation}

We conducted a total of $10^5$ rounds, providing the adversary with the flexibility to corrupt any $k$ out of the $T$ rounds, with the choice made uniformly at random. The experiments covered scenarios where $k$ assumed values from the set $\Big\{0, \sqrt{T}, T^{\frac{2}{3}}, \frac{T}{4}\Big\}$. In the context of corrupted rounds, $y_t$ was  chosen uniformly at random from the interval $[0, 1]$.

\paragraph{Strongly Convex Loss Function}

We employed the following fixed loss function $f_t$ for each round:
\begin{align}
\label{eq:online lin reg loss}
f_t\big( (x_t, y_t), \theta \big) = \frac{\lambda}{2} \| \theta \|^2 +  (y_t - \inner{x_t}{\theta})^2~.
\end{align}

The parameter $\lambda$ was set to $10^{-4}$. To enforce an unbounded domain for $\theta$, we considered it to be an element of $\real^{100}$. In each round, the comparator $\theta_t^*$ was selected by minimizing the function $f_t((x_t, y_t), \theta)$ as defined in Equation \eqref{eq:online lin reg loss}. Although it may not precisely match the ground truth $\theta^*$ due to the low influence of additive noise, our experimental results indicate its proximity to $\theta^*$.

\paragraph{Choice of Parameters}

In the context of \Learn{}, the parameter $a$ was set to $10$, and $b$ was held at $10$. Once again, this choice resulted from a grid search across a limited range of values, with no dedicated tuning efforts for optimality. Consistently, across all methods, including both baselines and \Learn{}, a uniform step size of $\alpha = \frac{1}{\sqrt{T}}$ was utilized.

\paragraph{Results}

The experiments were executed across 30 independent runs, and the outcomes are presented in Figure~\ref{fig:  Clean regret for lin reg}. The results show the similar behavior as in the case of online SVM. We recall them here in the context of ridge regression:

\begin{itemize}
    \item Similar to Section~\ref{subsec:online-svm}, \Learn{}  again employs a cautious update strategy, evident in the gap in clean regret in Figure \ref{fig: lin reg a}. Unlike other baselines, \Learn{}  accumulates initial regret due to its cautious update mechanism designed for anticipating corrupted rounds, leading to slow flattening of the regret curve.
    \item In the presence of outliers (until Figure~\ref{fig: lin reg d}), \Learn{} remains robust, maintaining a maximum regret around $1000$, while other baselines experience a sharp surge in regret.
    \item In our experiments, due to low noise levels, \Learn{} captures the underlying ground truth, sustaining a flat region in Figures~\ref{fig: lin reg a} to \ref{fig: lin reg c} even with outliers. This contrasts with other methods struggling to achieve sustained flatness. OGD, in particular, proves susceptible to outlier influence.
    \item A constant proportion of outliers adversely affects all methods, as expected from Theorem~\ref{thm: clean dynamic regret bound}. \Learn{} starts accumulating increasing regret when $k = \frac{T}{4}$, depicted in Figure~\ref{fig: lin reg d}, with all baseline methods exhibiting poor performance in this regime.
\end{itemize}

\paragraph{Discussion on the choice of the step size for baselines}

To clarify our choice of step size in the experiments: while it is theoretically possible to set $\alpha_t = \frac{1}{mt}$ given that the losses are $m$-strongly convex, we encountered a situation where the value of $m$ was extremely small ($m = \lambda = 10^{-4}$). This led to an excessively large step size, causing the anticipated regret bound of $\calO(\frac{\log T}{m})$ to become quite substantial. Upon numerical evaluation, we found that using $\alpha = \frac{1}{\sqrt{T}}$ produced a significantly smaller regret for the baselines. As a result, we adopted this more favorable step size in our experimental framework.

\newpage

\end{document}